\documentclass{article} 
\usepackage{iclr2027_conference,times}

\usepackage{amsmath,amsfonts,bm}

\def\eqref#1{equation~\ref{#1}}

\def\1{\bm{1}}

\DeclareMathAlphabet{\mathsfit}{\encodingdefault}{\sfdefault}{m}{sl}
\SetMathAlphabet{\mathsfit}{bold}{\encodingdefault}{\sfdefault}{bx}{n}

\usepackage{hyperref}

\usepackage{amsmath}
\usepackage{amssymb}
\usepackage{mathtools}
\usepackage{amsthm, thm-restate}
\usepackage{algorithm}
\usepackage{comment}
\usepackage{subcaption}
\usepackage{booktabs}
\usepackage{xcolor}

\usepackage{amsmath,amssymb}
\usepackage{algorithm}
\usepackage{algpseudocode}

\usepackage{enumitem}

\usepackage{xspace}
\makeatletter
\DeclareRobustCommand\onedot{\futurelet\@let@token\@onedot}
\def\@onedot{\ifx\@let@token.\else.\null\fi\xspace}

\def\eg{{e.g}\onedot}

\makeatother

\usepackage[capitalize,noabbrev]{cleveref}

\theoremstyle{plain}
\newtheorem{theorem}{Theorem} 
\newtheorem{proposition}[theorem]{Proposition}

\theoremstyle{definition}
\newtheorem{definition}{Definition}   

\theoremstyle{remark}

\newcommand{\method}{DP-$\lambda$CGD\xspace}

\title{\method: Efficient Noise Correlation for 
  Differentially Private Model Training}

\iclrfinalcopy 

\author{%
  Nikita P. Kalinin\\
  Institute of Science and Technology Austria \\
  Klosterneuburg, Austria \\
  \texttt{nikita.kalinin@ist.ac.at}
  \And
  Ryan McKenna \\
  Google \\
  Seattle, US\\
  \texttt{mckennar@google.com}
  \And
  Rasmus Pagh \\
  University of Copenhagen \\
  Copenhagen, Denmark \\
  \texttt{pagh@di.ku.dk}\\
  \And
  Christoph H. Lampert \\
  Institute of Science and Technology Austria \\
  Klosterneuburg, Austria\\
  \texttt{chl@ist.ac.at}
}

\begin{document}

\maketitle

\begin{abstract}
Differentially private stochastic gradient descent (DP-SGD) is the gold standard for training machine learning models with formal differential privacy guarantees. Several recent extensions improve its accuracy by introducing correlated noise across training iterations. Matrix factorization mechanisms are a prominent example, but they can require storing previously added noise vectors, leading to substantial memory overhead. In this work, we propose a new algorithm for correlated-noise private training that correlates noise only with the immediately preceding iteration and cancels a tunable portion of it. By regenerating the previous noise using a pseudorandom number generator rather than storing it, our method has essentially the same memory footprint as DP-SGD. We show that the computational overhead is minimal and empirically demonstrate improved accuracy over DP-SGD.
\end{abstract}

\section{Introduction}
Protecting individuals’ privacy is a fundamental constraint for algorithms that interact with humans or are trained on personal data. The disclosure of sensitive information can harm both users and the reputation of organizations deploying such algorithms. In machine learning, this challenge is particularly acute as the demand for non-public and user-generated data continues to grow. To mitigate these risks, \textit{Differential Privacy} \citep{Dwork_DP_original} has emerged as a mathematically formal framework for privacy protection. By introducing noise into algorithms or model training procedures, it guarantees protection for individuals whose data are used, while still enabling meaningful utility for the resulting models.

The most well-studied approach for introducing differential privacy into model training is differentially private stochastic gradient descent (DP-SGD) \citep{abadi2016deep}. DP-SGD clips per-example gradients to bound the influence of any single data point and then adds independent Gaussian noise to ensure privacy. One of the principal ways to improve the utility of this method is to introduce correlations in the noise across training iterations. Mechanisms that correlate noise in this manner are generally referred to as matrix factorization mechanisms \citep{li2015matrix, kairouz2021practical, denisov2022improved, choquette2022multi, choquette2023amplified, kalinin2024, kalinin2025learningrate}. Such mechanisms have been successfully applied in real-world systems, including large-scale deployments such as Google Gboard \citep{xu2024advances}. 

Matrix factorization mechanisms, however, have inherent limitations. To introduce correlations across iterations, noise from previous steps must be accessible and is therefore typically stored. Because each noise vector has the same dimensionality as the (trainable) model parameters, this can lead to prohibitive memory overhead in some settings. For example, the original work of \citet{choquette2022multi} on multi-epoch training required storing as many noise vectors as there are iterations. Subsequent work by \citet{scalingmckenna2024} showed that the memory overhead of matrix factorization can be reduced; in particular, in multi-GPU settings with large batch sizes, the time overhead is negligible. However, for larger models trained on fewer GPUs, the overhead is still substantial. More recently, \citet{kalinin2025back} demonstrated that improvements over DP-SGD can be achieved by storing as few as four noise vectors. In contrast, we push this line of work to its logical limit by eliminating the need for any additional memory. Our method leverages noise regeneration: since noise is almost always produced by a pseudo-random generator, previously used noise vectors can be deterministically regenerated when needed rather than stored explicitly.

Another important question concerns how to correlate the noise, or more precisely, which matrix yields the best performance. The most commonly used metric for evaluating the quality of a factorization is the root mean squared error (RMSE) \citep{denisov2022improved, choquette2022multi, scalingmckenna2024, kalinin2025back}. Maximum squared error (MaxSE) has been proposed as an alternative and is more commonly used in settings such as single-participation or item-level privacy \citep{henzinger2023almost, henzinger2025improved, henzinger2025normalized, andersson2024streaming, jacobsen2025private}. However, it has been observed that optimizing for these metrics does not perfectly predict the utility of the resulting factorization \citep{kalinin2024, kalinin2025back}, and this mismatch is identified as a major open problem in matrix factorization mechanisms in a recent survey \citep{pillutla2025correlated}.

Our proposed method enables continuous interpolation between a trivial factorization corresponding to standard DP-SGD and factorizations that are optimal with respect to RMSE or MaxSE, controlled by a single hyperparameter. This flexibility allows us to fully explore the space of factorizations within the considered matrix class and to systematically study the problem of optimal matrix factorization. Our results reaffirm prior findings \citep{koloskova2023gradient, choquette2023correlated, gu2025correlating} that the optimal factorization is problem- and loss-dependent, and cannot be determined solely from the matrix structure and training configuration. Consequently, we treat the optimal parametrization of our mechanism as a single tunable hyperparameter, which can be lightly adjusted to match the specific learning task.

\paragraph{Contribution}
We introduce \method, a memory-efficient algorithm for correlated noise private training that has essentially the same memory usage as DP-SGD and incurs only small computational overhead.
We systematically explore the one-parameter family induced by \method and show that RMSE and MaxSE optimal choices of $\lambda$ do not necessarily maximize downstream accuracy.
We theoretically analyze this family, deriving error guarantees and structural properties of RMSE and MaxSE optimal factorizations.
Finally, we show empirically that \method improves over DP-SGD and is competitive with substantially higher bandwidth matrix factorization methods, while retaining the efficiency of bandwidth $2$ noise correlation.

\section{Background}

We use \emph{differential privacy}~\citep{Dwork_DP_original} as our privacy
notion and recall its definition below.

\begin{definition}[Differential Privacy (DP) \citep{Dwork_DP_original}]
A randomized algorithm $M$ is said to provide $(\epsilon, \delta)$-differential
privacy if, for all neighboring datasets $D$ and $D'$, and for all measurable
subsets $S$ of the mechanism's output space:
\begin{equation}
\Pr[M(D) \in S] \leq e^{\epsilon} \cdot \Pr[M(D') \in S] + \delta.
\end{equation}
\end{definition}
Throughout, by neighboring datasets we mean the standard zero-out notion used for
matrix mechanisms, where one data point is replaced by a null element.

To ensure that a model is differentially private, the most common approach is to use differentially private stochastic gradient descent (DP-SGD)~\citep{abadi2016deep}. At each iteration $t$, a batch of data $B_t$ is sampled, and gradients
$g_i^t$ are computed for all $i \in [1, |B_t|]$. These gradients are then clipped to a fixed
$\ell_2$-norm bound $\zeta > 0$ and averaged to obtain the clipped gradient
$\hat{g}_t$. Finally, scaled Gaussian noise $z_t$ is added to ensure that the resulting
update satisfies a prescribed level of privacy.

The utility of this mechanism can be improved by correlating the noise added across
iterations. Intuitively, one may add larger noise at each iteration while choosing the
correlation so that components of previously added noise cancel over time. As a result,
the accumulated noise can have smaller variance, improving model accuracy. A formal way
to introduce linear noise correlation is via a \emph{correlation matrix}\footnote{We 
 refer to the \emph{correlation} matrix as $C^{-1}$, and to $C$, which determines 
the noise scale, as the \emph{strategy} matrix.}
$C^{-1}$~\citep{li2015matrix, denisov2022improved, choquette2022multi}. If we stack all gradient vectors into
a matrix $G$ and all noise vectors into a matrix $Z$, then instead of adding $Z$
directly to $G$, we add correlated noise $C^{-1}Z$ and obtain $G + C^{-1}Z$.
By post-processing invariance, this is equivalent to adding noise $Z$ to the transformed
gradients $CG$. Ignoring privacy amplification due to subsampling for now, the amount of
noise required by a matrix factorization mechanism is proportional to the
\textit{sensitivity} of $CG$:
\begin{equation}
    \mathrm{sens}(C)
    \;=\;
    \sup_{G \sim G'} \|CG - CG'\|_F,
\end{equation}
where $G$ and $G'$ differ in all positions corresponding to the participation of a single
data point.

Following the framework of \citep{choquette2023amplified}, we consider a multi-epoch training setting in which each data point can participate up to $k$ times, under the restriction that the time gap between consecutive participations is at least $b$. This framework can be directly transferred to the federated learning setting, where each user participates at most $k$ times, and we do not include a user’s input unless we receive sufficient input from the rest of the users.

For the remainder of this work, we restrict attention to positive lower triangular
Toeplitz (LLT) strategy matrices. Under the above $b$-min-separation condition, the
sensitivity of such a matrix can be computed using the following result.
\begin{theorem}[Theorem~2 from \citet{kalinin2024}]
\label{thm:sensitivity}
Let $C$ be a lower triangular Toeplitz matrix with decreasing non-negative entries
$c_0 \ge c_1 \ge \cdots \ge c_{n-1} \ge 0$.
Then, under $b$-min-separation, the sensitivity is given by
\begin{equation}
    \mathrm{sens}_{k,b}(C)
    = \Big\| \sum_{j=0}^{k-1} C_{[:,\,jb + 1]} \Big\|_2,
\end{equation}
where $C_{[:,\,jb + 1]}$ denotes the $(1+jb)$-th column of $C$.
\end{theorem}

To achieve a target privacy level $(\epsilon,\delta)$, we add Gaussian noise whose scale
is proportional to the sensitivity, multiplied by the Gaussian mechanism noise multiplier (the
multiplier can be computed numerically; see \citet{balle2018improving}). With this
calibration, using correlated noise induced by $C$ attains the desired privacy
guarantee. The main remaining question is therefore how to choose the correlation (or
strategy) matrix $C$. This turns out to be a subtle and decidedly nontrivial problem,
and it is the core focus of the matrix factorization literature; see the recent survey
of \citet{pillutla2025correlated}.

As an efficiently computable proxy for solution quality, we use the Root Mean Squared
Error (RMSE) and the Maximum Squared Error (MaxSE) between the true and noisy sequences
of intermediate models.
Let $G \in \mathbb{R}^{n \times d}$ be the matrix whose $t$-th row is the (vectorized)
gradient at iteration $t$. Let $A \in \mathbb{R}^{n \times n}$ be the \emph{prefix-sum} matrix,
a lower triangular matrix defined by $A_{ij} = \mathbf{1}\{i \ge j\}$.
Then $AG$ stacks the cumulative sums of gradients, and therefore corresponds (up to the
learning-rate factor and initialization) to the sequence of intermediate model updates.

Rather than estimating $AG$ directly, we form a noisy estimate by first perturbing the
per-iteration gradients with correlated noise and then taking prefix sums:
\begin{equation}
    \widehat{AG} = A(G + C^{-1}Z) = AG + AC^{-1}Z,
\end{equation}
where $Z \in \mathbb{R}^{n \times d}$ has i.i.d.\ standard Gaussian entries.

Natural measures of estimation quality include the expected squared Frobenius error
$\mathbb{E}_{Z}\!\left[\|AC^{-1}Z\|_F^2\right]$ and the worst-case (over iterations) expected
squared row error $\max_{i \in [n]} \mathbb{E}_{Z}\!\left[\|(AC^{-1}Z)_{i,:}\|_2^2\right]$.
Let the left factor be $B := AC^{-1}$. For each error measure, we drop multiplicative constants that are
independent of the choice of factorization $(B,C)$ (e.g., dimensionality $d$, clipping norm $\zeta$, noise multiplier $\sigma_{\epsilon, \delta}$), and work with the following equivalent objectives:
\begin{equation}
\label{eq:RMSE_definition}
    \mathrm{RMSE}(B, C)
    \;:=\;
    \tfrac{1}{\sqrt{n}}\,\|B\|_F \cdot \mathrm{sens}_{k,b}(C), \qquad 
    \mathrm{MaxSE}(B, C)
    \;:=\;
    \|B\|_{2 \to \infty} \cdot \mathrm{sens}_{k,b}(C),
\end{equation}
where $\|B\|_{2 \to \infty} := \max_{i \in [n]} \|B_{i,:}\|_2$.

Matrix Factorization mechanisms, as well as DP-SGD, can benefit from privacy amplification. The idea is that each training batch is selected at random (via shuffling or a more intricate scheme), creating uncertainty for an adversary about whether a particular data point was included in a given batch. With appropriate analysis, via so called privacy accountants, this ambiguity allows one to reduce the amount of noise added to gradients while maintaining the same privacy guarantees.

In this work, we use Balls in Bins subsampling and the corresponding accountant for matrix factorization \citep{choquette2024near}. In this subsampling scheme, one first fixes the number of batches in an epoch and then randomly allocates each data point to one of the batches. The resulting allocation is reused across epochs, providing computational efficiency. See Subsection~\ref{sub:amplification} for more details on amplification.

\section{Differentially Private Correlated Gradient Descent (\method)}

In this paper, we propose \method, a matrix factorization and noise-correlation technique
that improves upon the utility of DP-SGD while remaining nearly as fast, requiring no
additional memory. Formally, we correlate the noise using a lower-triangular Toeplitz
strategy matrix $C_\lambda$ and its inverse (the corresponding correlation matrix)
$C_\lambda^{-1}$, defined entrywise as follows. For $i,j \in [n]$,
\begin{align}
\label{eq:C_lambda}
    (C_\lambda)_{ij}
    \;:=\;
    \begin{cases}
        \lambda^{\,i-j}, & i \ge j,\\
        0, & i < j,
    \end{cases}
    \qquad(C_\lambda^{-1})_{ij}
    \;:=\;
    \begin{cases}
        1, & i=j,\\
        -\lambda, & i=j+1,\\
        0, & \text{otherwise}.
    \end{cases}
\end{align}
Correlating the noise using $C_{\lambda}^{-1}$ is equivalent to subtracting a
$\lambda$-fraction of the noise added at the previous iteration (up to an overall
scaling determined by $\mathrm{sens}_{k,b}(C_{\lambda})$).

To make our method memory-free, we observe that at each step we only need the noise
vector added in the previous iteration. As in essentially all practical private learning
pipelines, this noise is generated by a pseudo-random number generator (PRNG). Thus, by
storing the PRNG state from the previous iteration, we can roll back to that state,
regenerate the previous noise for cancellation, and then generate fresh noise exactly as
DP-SGD would. The drawback is that we must regenerate an additional noise vector each iteration, which
could introduce overhead. Empirically, however, this cost is negligible compared to the
rest of training (e.g., gradient computation): across multiple settings and model
families, \method runs within $2\%$ of the runtime of DP-SGD for the larger
models in our evaluation. See
Section~\ref{sec:noise_regeneration} for details.

Correlated noise can be combined with privacy amplification by subsampling. We use the
Balls-in-Bins subsampling scheme and its corresponding accountant \citep{choquette2024near}
to compute the resulting privacy guarantees. This allows improved utility at the same
target privacy level. For further discussion of the role of subsampling, see
Subsection~\ref{sub:amplification}.

Putting these ideas together, we outline the full procedure in
Algorithm~\ref{alg:method}. The algorithm takes $\lambda \in [0,1)$ as a hyperparameter; we discuss how to choose $\lambda$ in Section~\ref{sec:choice_of_lambda}.

We establish several properties of the \method with respect to the RMSE and MaxSE error
metrics. We begin by analyzing the MaxSE error, which is characterized in the following
lemma.

\begin{restatable}{lemma}{MaxSEBound}
\label{lem:MaxSEBound}
For $n=kb$, the MaxSE and RMSE of \method, optimized over $\lambda \in [0,1)$, satisfy
\begin{equation}
\begin{aligned}
    \!\!\inf_{\lambda \in [0,1)}\mathrm{MaxSE}(AC_{\lambda}^{-1}, C_{\lambda})= \Theta \big(k + \sqrt{k}\, n^{1/4}\big),\\
    \!\!\inf_{\lambda \in [0,1)}\mathrm{RMSE}(AC_{\lambda}^{-1}, C_{\lambda})= \Theta \big(k + \sqrt{k}\, n^{1/4}\big).
\end{aligned}
\end{equation}
\end{restatable}
The $\Theta(n^{1/4})$ term for the single-participation setting, as well as the optimal factorization error $\Theta(k+\sqrt{k}\log n)$, were derived in \citet{kalinin2025back}. Here, we extend the factorization result to the multi-participation regime. Our bound improves over the trivial (i.e., DP-SGD) factorization, which yields error $O(\sqrt{nk})$, though it remains weaker than the optimal factorization error. Despite this, we will show that improved RMSE does not always translate into better model
performance. In particular, in certain regimes, \method outperforms factorizations with
smaller RMSE.

The \method is the first framework that allows us to establish structural properties of
optimal factorizations. In particular, we show that the minimizer of the MaxSE objective is
always larger than the minimizer of the RMSE objective, which makes RMSE a more practically
relevant metric. We discuss this phenomenon in detail in
Section~\ref{sec:choice_of_lambda}.

\begin{restatable}{lemma}{MaxSEvsRMSEMinimizerInequality}
\label{lem:method_minimizer_inequality}
For \method, a minimizer of MaxSE is always at least as large as a minimizer of RMSE, both with and without amplification by subsampling.
\end{restatable}
\paragraph{Relation of $C_{\lambda}$ matrix to prior works.}
We note that the correlation matrix $C_{\lambda}$ has appeared as an exemplary factorization in several prior works. It is referred to as ``Anti-PGD + damping'' in \citet{choquette2023correlated}, and it is used to illustrate noise correlation in the recent survey of \citet{pillutla2025correlated}. It can also be viewed as a special case of the BandInvMF factorization of \citet{kalinin2025back}. However, it has not previously been used for DP training, $\lambda$ has not been treated as a tunable hyperparameter, and its computational efficiency in terms of time and memory has not been evaluated.

\section{Noise Regeneration}
\label{sec:noise_regeneration}

\begin{figure}[t]
    \centering

    \begin{subfigure}[t]{0.48\linewidth}
        \centering
        \includegraphics[width=\linewidth]{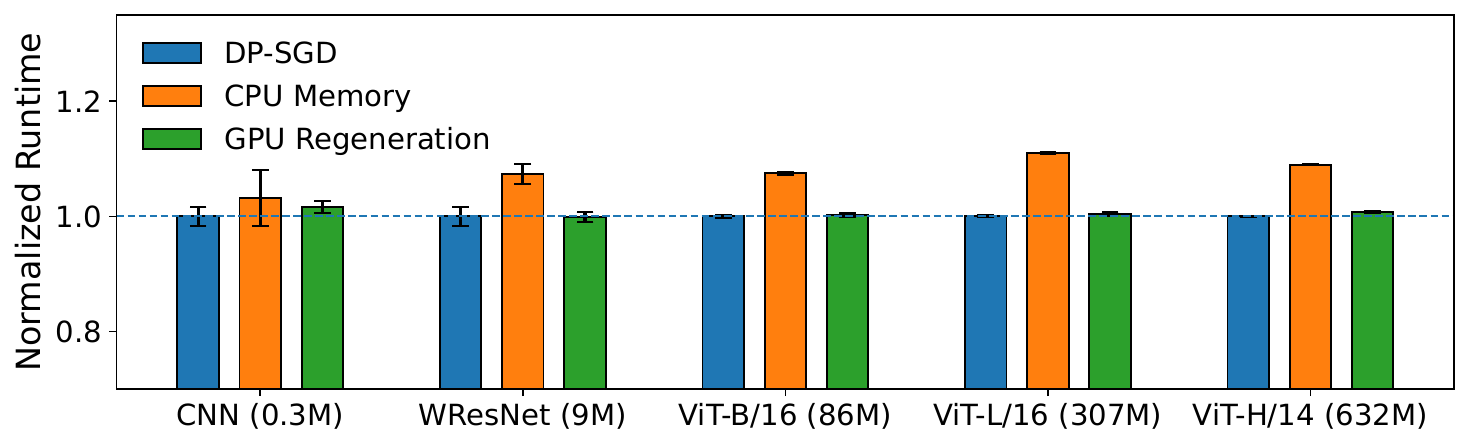}
        \caption{CIFAR-10 runtime for one epoch.}
        \label{fig:time_comparison_cifar}
    \end{subfigure}
    \hfill
    \begin{subfigure}[t]{0.48\linewidth}
        \centering
        \includegraphics[width=\linewidth]{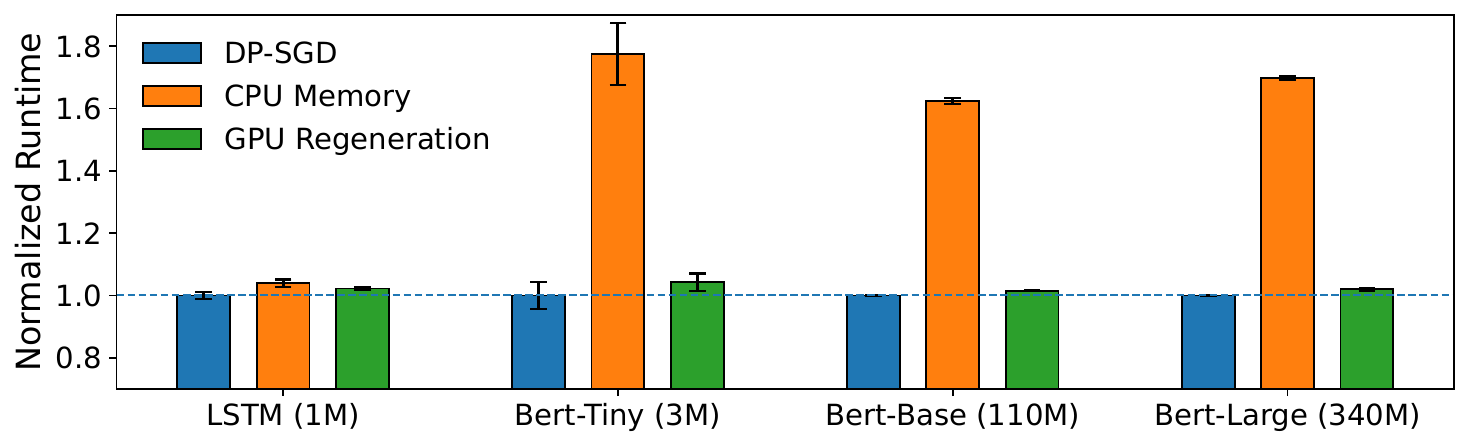}
        \caption{IMDB fine-tuning runtime for one epoch.}
        \label{fig:time_comparison_imdb}
    \end{subfigure}

    \caption{Runtime for one epoch, comparing DP-SGD runtime using the Opacus implementation versus storing the noise on the CPU and regenerating it on the GPU using pseudorandomness.}
    \label{fig:time_comparison}
\end{figure}

Modern private machine learning uses pseudo-random number generators (PRNGs) to generate noise.
A PRNG has a small \emph{state} that determines its next output as well as its next state.
The state is typically initialized to a random ``seed'' value using hardware-supported random number generation, and the whole sequence of random values is determined by the initial state.
So-called \emph{counter-based} PRNGs have a fixed state given by a seed $s$ and take a counter $i$ as input to generate the $i$th random value (that is, it implements a hash function from counters to random values).
PyTorch uses the counter-based {\tt Philox4x32-10} pseudorandom number generator~\cite{salmon2011parallel} on CUDA devices, which deterministically maps a 64-bit seed and a 128-bit counter to random values, enabling reproducible and scalable parallel random number generation on GPUs.

\paragraph{Regeneration.}
Previous approaches to noise correlation  required storing noise vectors generated in earlier steps, which can require substantial memory that may not be available on the GPU. 
Alternatively, one can store the noise on the CPU and transfer it to the GPU, but this incurs substantial overhead.
By storing a state of the PRNG it is possible to go back and ``replay'' any part of the pseudorandom sequence, trading the space for the time needed to regenerate randomness.
This approach was believed to be impractical because previous factorizations had non-banded correlation matrices and required regenerating the entire noise sequence; see the discussion in \citet{kim2025cocoon}. We also note that the reference code of \citet{choquette2022multi} implements such “on-the-fly” noise.\footnote{\url{https://github.com/google-research/federated/blob/master/dp_matrix_factorization/matrix_factorization_query.py\#L190}}
 However, their implementation reconstructs the entire noise sequence from the first iteration at each step, which they described as computationally expensive.
 
Our proposed \method, as well as Banded Inverse Matrix Factorization and Banded Inverse Square Root with small bandwidth \citep{kalinin2025back}, needs to access only a small number of the most recent noise vectors. 
This makes it efficient to regenerate noise on the fly without storing the history. Note that banded methods such as Banded Matrix Factorization \citep{scalingmckenna2024} and Buffered Linear Toeplitz \citep{dvijotham2024efficient, hasslefree2024} can correlate noise using a relatively small in-memory buffer; however, at each step the correlated noise depends on components originating from the very first iteration. Regenerating the full sequence at every step would be impractical (see Table~\ref{tab:noise_time_buffer} for the time and space requirements).

\textbf{Discussion.}
The random values generated by Philox and other popular high-performance PRNGs used in machine learning frameworks pass statistical tests of randomness but are not cryptographically secure. This potentially weakens the privacy protection, but since the de-facto standard is to not use slower, cryptographically secure PRNGs~\cite{egan2024highspeed} we will follow this standard.
Using slower PRNGs may shift the trade-off between the cost of storing noise values and regenerating them.

\textbf{Evaluation.}
We evaluate GPU-based noise regeneration against DP-SGD and an alternative that stores and correlates noise on the CPU. In Fig.~\ref{fig:time_comparison}, we report the relative time per epoch on CIFAR-10 with batch size $512$, for models ranging from a $0.3$M-parameter CNN to ViT-H/14 with $632$M parameters \citep{dosovitskiy2021vit}. We also evaluate on IMDB using an LSTM and a family of BERT models \citep{devlin2019bert}, with batch size $128$. For larger models, GPU regeneration adds less than $1\%$ overhead on CIFAR-10 and less than $2\%$ for BERT-Base and BERT-Large on IMDB. The overhead is slightly higher for faster runs, reaching $1.6\%$ for the CNN and $4.2\%$ for BERT-Tiny. In contrast, storing noise on the CPU and transferring it to the GPU incurs substantially larger overhead, reaching $11\%$ on CIFAR-10 and $78\%$ on IMDB.

\begin{figure}[t]
    \centering
    
    \begin{subfigure}{0.32\linewidth}
        \centering
        \includegraphics[width=\linewidth]{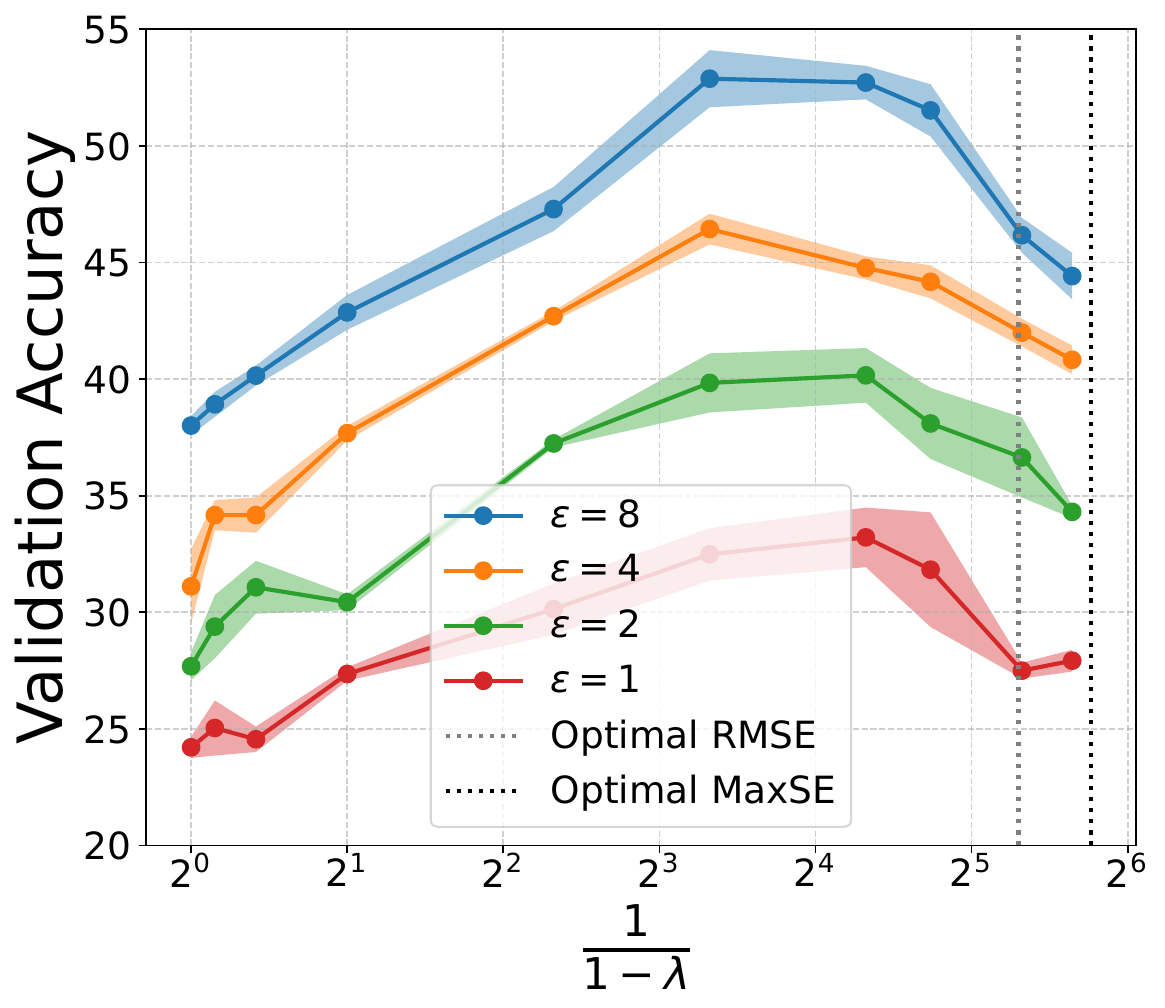}
        \caption{CIFAR-10: $B= 128$, $k =10$}
    \end{subfigure}
    \begin{subfigure}{0.32\linewidth}
        \centering
        \includegraphics[width=\linewidth]{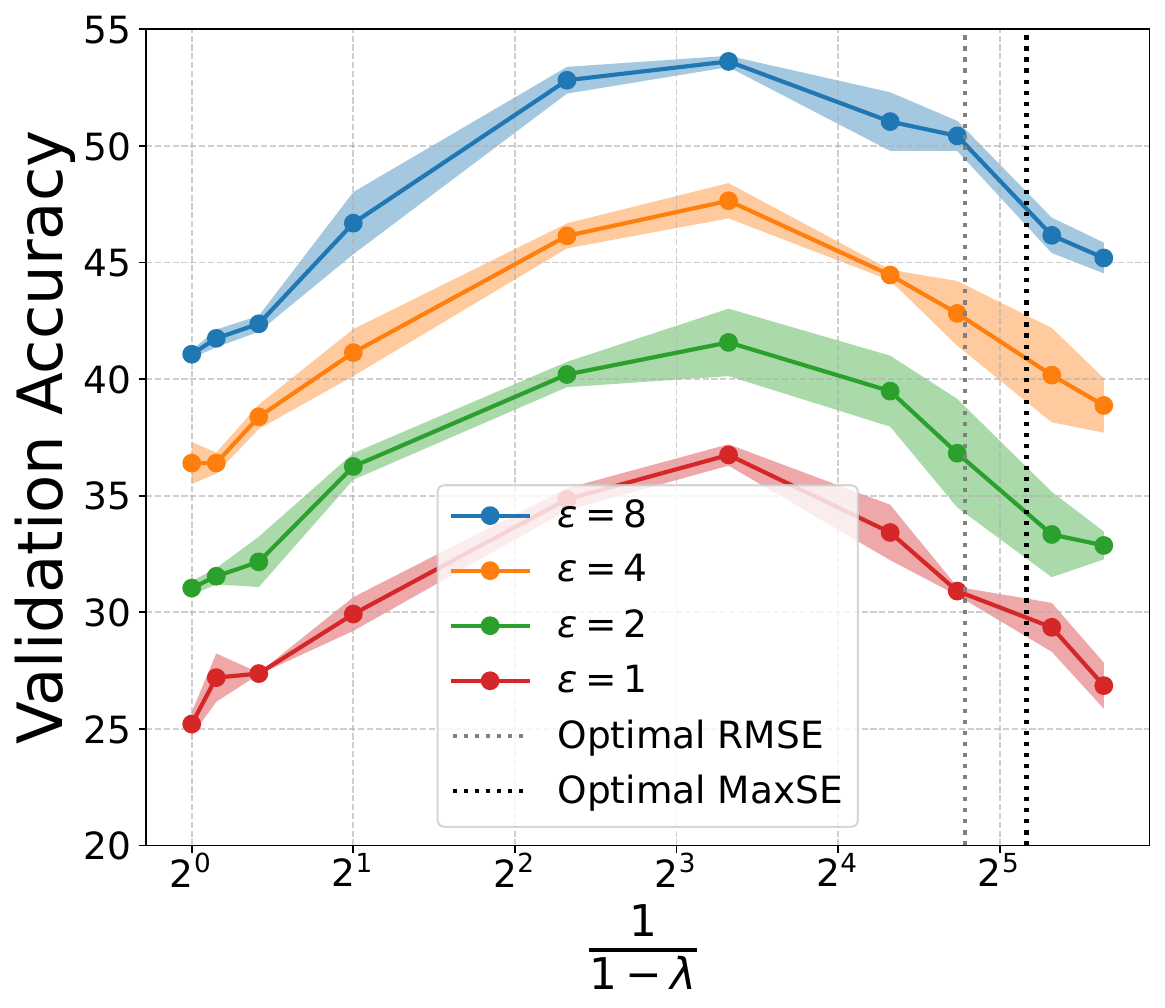}
        \caption{CIFAR-10: $B=256$, $k=10$}
    \end{subfigure}
    \begin{subfigure}{0.32\linewidth}
        \centering
        \includegraphics[width=\linewidth]{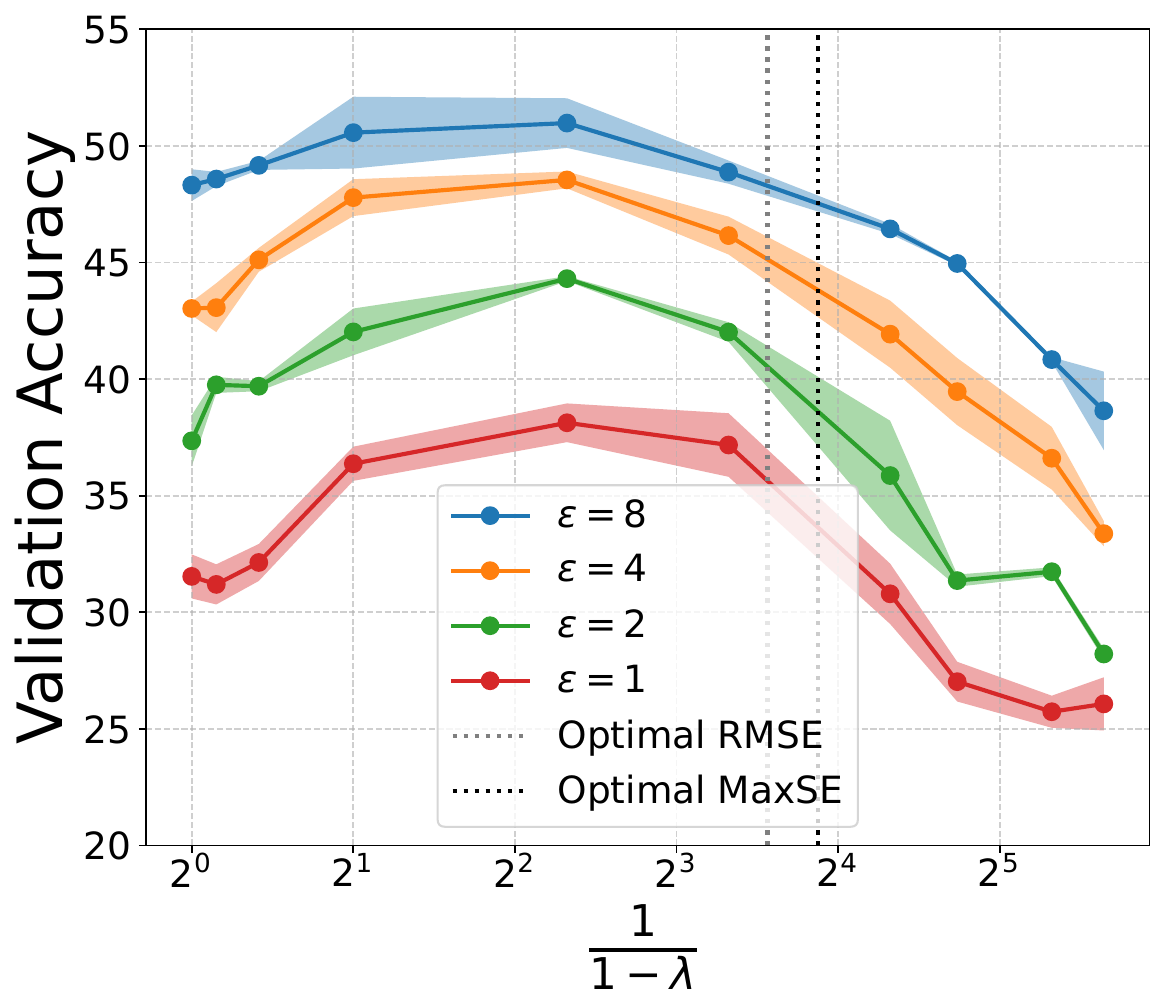}
        \caption{CIFAR-10: $B=1024$, $k=10$}
    \end{subfigure}

    \begin{subfigure}{0.32\linewidth}
        \centering
        \includegraphics[width=\linewidth]{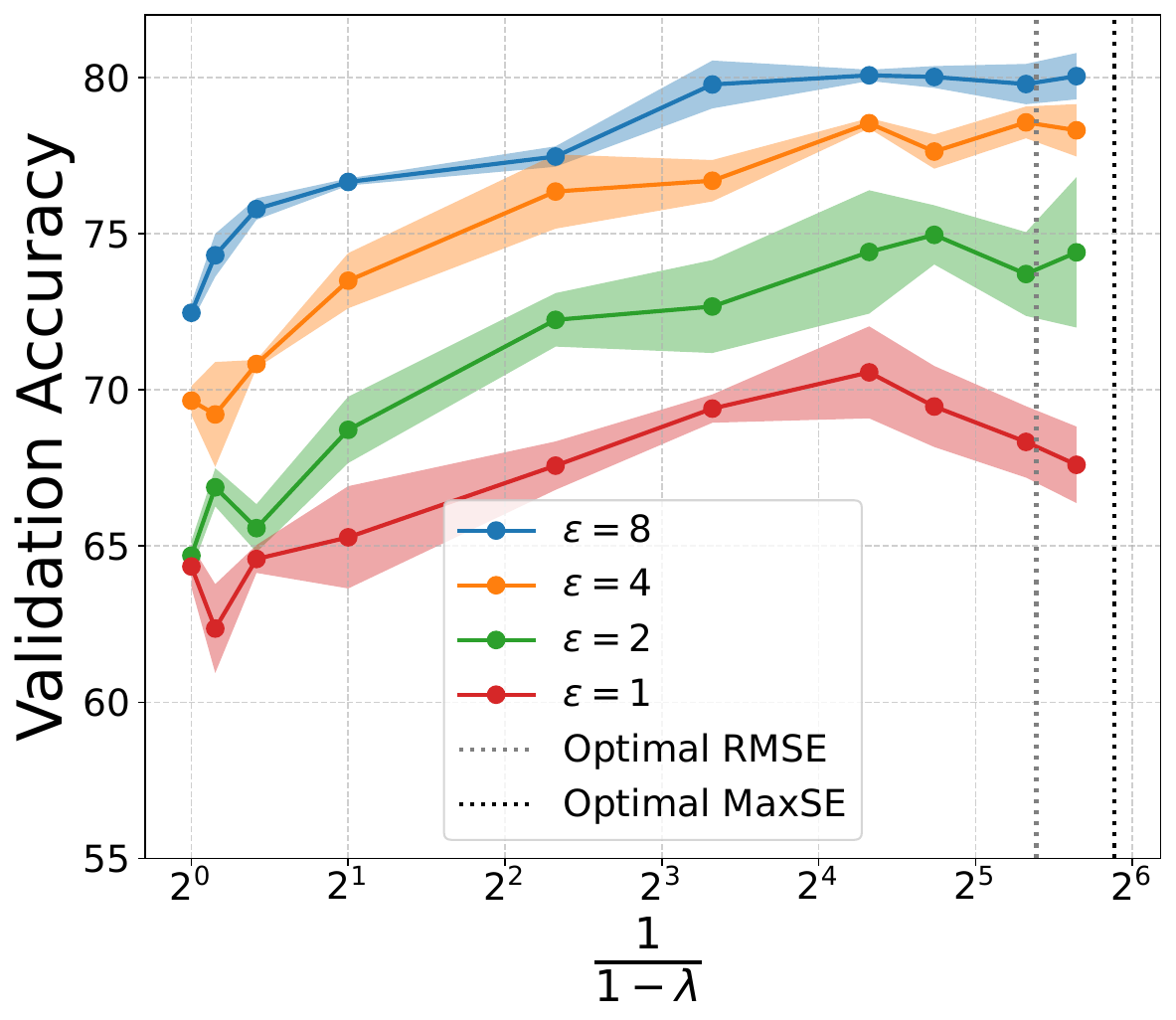}
        \caption{IMDB: $B=64$, $k=10$}
    \end{subfigure}
    \begin{subfigure}{0.32\linewidth}
        \centering
        \includegraphics[width=\linewidth]{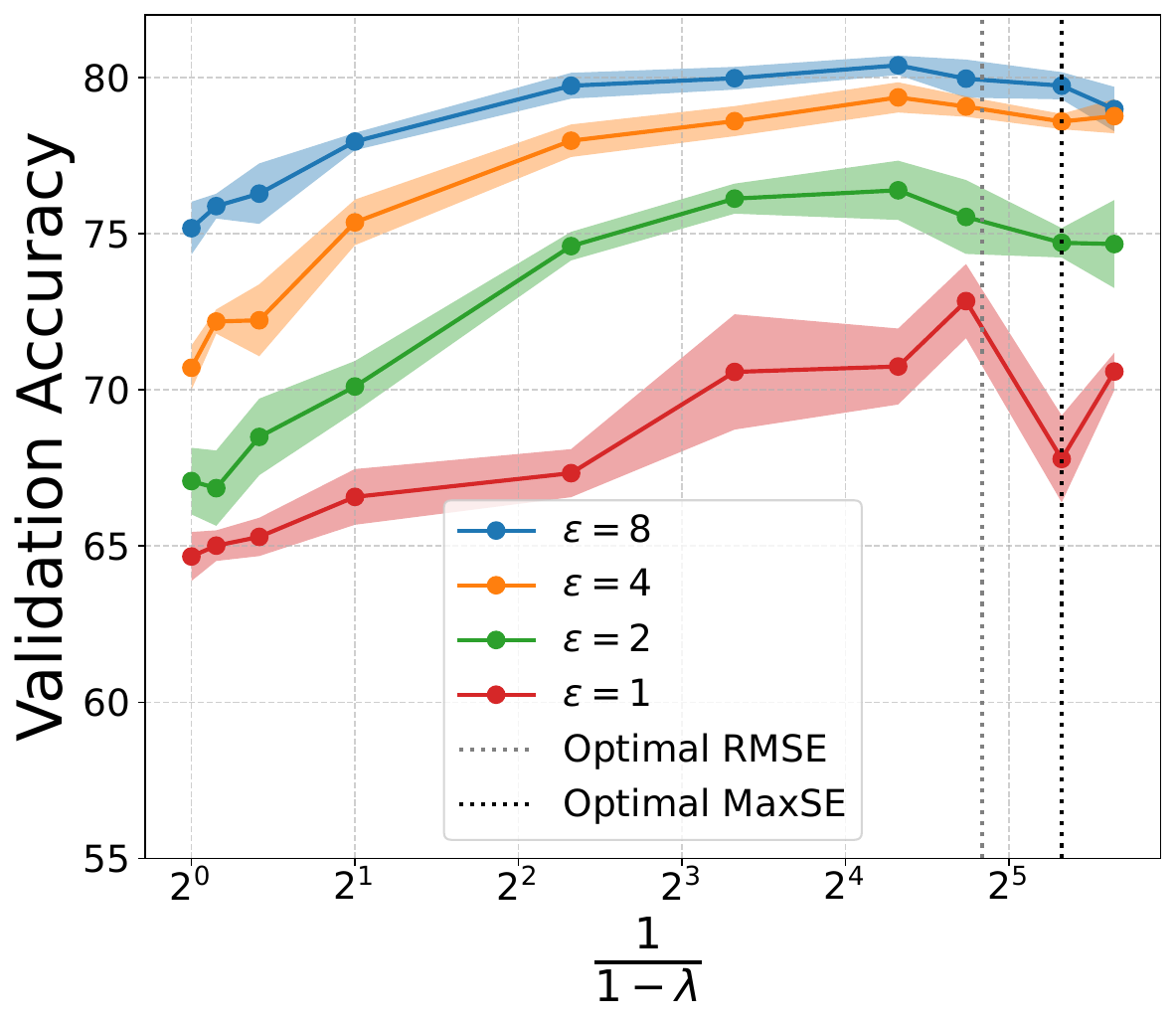}
        \caption{IMDB: $B=128$, $k=10$}
    \end{subfigure}
    \begin{subfigure}{0.32\linewidth}
        \centering
        \includegraphics[width=\linewidth]{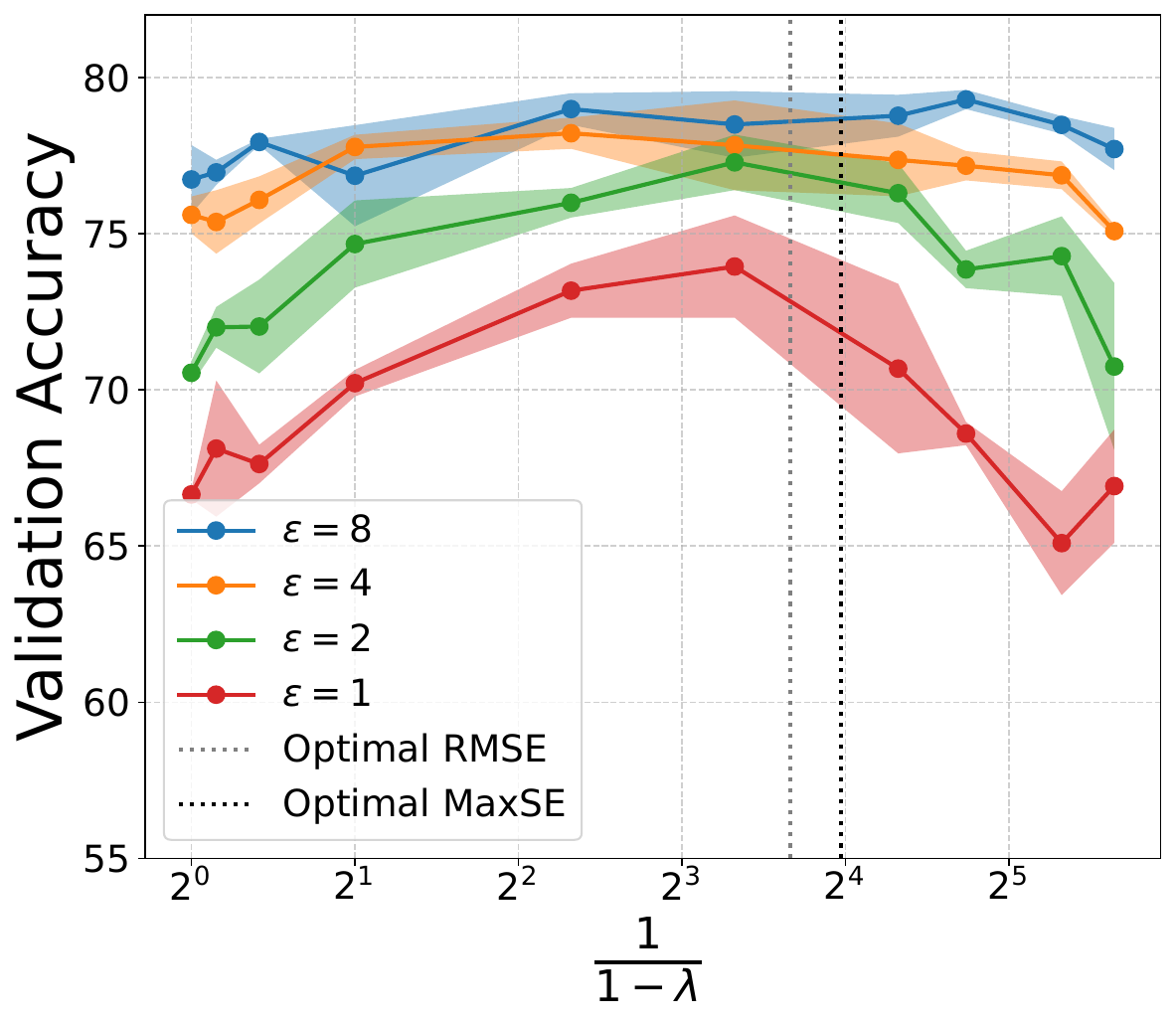}
        \caption{IMDB: $B=512$, $k=10$}
    \end{subfigure}

    \caption{ Validation accuracy for different values of the parameter $\lambda$ in \method with different privacy budgets $\epsilon$ and $\delta=10^{-5}$.
The first row shows a CNN model trained on CIFAR-10 for $k=10$ epochs, using different batch sizes $B$.
The second row shows fine-tuning of the BERT-tiny model on the IMDB sentiment analysis dataset.
}
    \label{fig:choice_of_lambda}
\end{figure}

In Tables~\ref{tab:runtime_comparison} and~\ref{tab:runtime_comparison_imdb} in the appendix, we report the absolute runtime for different bandwidths, where bandwidth $2$ corresponds to \method and larger bandwidths correspond to Banded Inverse Matrix Factorization \citep{kalinin2025back}. 
For bandwidth $4$, GPU regeneration adds at most about $9\%$ overhead across all models. For bandwidth $16$, the overhead remains below $11\%$ except for BERT-Tiny, where it reaches $45\%$. In contrast, storing and transferring noise on the CPU can increase runtime by up to $147\%$.

\section{How to choose the value $\lambda$?}
\label{sec:choice_of_lambda}

The best value of $\lambda$ for \method can depend on the batch size, the number of
epochs, the target privacy level, whether amplification by subsampling is used, and,
ultimately, the learning problem (\eg, the loss function and data distribution). 
To study how to choose this parameter in practice, we report results of an extensive sweep over $\lambda$ in Figure~\ref{fig:choice_of_lambda}. 

Figure~\ref{fig:choice_of_lambda} reports results \emph{without} amplification, which
allows us to characterize the behavior of RMSE- and MaxSE-optimal factorizations
theoretically; for the role of subsampling, see Subsection~\ref{sub:amplification}.
We make three empirical observations. 

\begin{enumerate}
    \item The best choice of $\lambda$ appears to be
task-dependent: under similar training settings and identical matrices, the optimal
$\lambda$ differs substantially between training a CNN on CIFAR-10 and fine-tuning
BERT-tiny on IMDB. In particular, for CIFAR-10 the optimal $\lambda$ is much smaller
than the value suggested by the RMSE-optimal factorization. Moreover, as shown in Lemma~\ref{lem:method_minimizer_inequality}, the minimizer of the MaxSE objective is always larger than the minimizer of the RMSE objective. Consequently, MaxSE is farther from the optimum and is therefore a less suitable objective for \method.
\item As the batch size increases, the RMSE-optimal factorization approaches
DP-SGD. However, Figure~\ref{fig:experiments_fixed_batch_size} in the appendix shows that, for a fixed non-full batch size, increasing the number of epochs pushes the optimal mechanisms further away from DP-SGD (e.g. the optimal $\lambda$ increases). This suggests that even in long training runs, common in high-performance DP-SGD training, the benefit of matrix factorization, and \method in particular, continues to grow.
\item The
optimal $\lambda$ is relatively insensitive to the privacy level across the range we
tested.

\end{enumerate}

\paragraph{The full-batch regime.}

We next provide theoretical support for the observed dependence of the optimal
$\lambda$ on the batch size. We consider the full-batch regime, corresponding to
$b=1$ and $k=n$, and study the multi-participation error in this limit. The optimal-factorization results of \citet{kalinin2025back} imply that DP-SGD,
corresponding to the trivial factorization, is asymptotically optimal in this regime,
with RMSE scaling as $O(n)$. This asymptotic result, however, does not determine
whether DP-SGD is optimal at the level of the leading constant.

For the MaxSE metric, this question has already been resolved:
\citet[Proposition~3.11]{pillutla2025correlated} show that DP-SGD is
constant-optimal in the full-batch regime. Their argument, however, does not extend to the RMSE objective. We therefore ask whether the same conclusion holds for RMSE.

The following theorem answers this question in the negative: DP-SGD is \textbf{not} constant-optimal under the RMSE metric. 
Instead, the weighted diagonal factorization with weights
$(n-j+1)^{1/4}$ achieves a smaller leading constant, reducing the error by
approximately $6\%$ compared with DP-SGD. In fact, this factorization is optimal\footnote{This improvement is obtained by injecting larger noise into later estimates, which may be undesirable in practice (see Remark 4.10 from \citet{pillutla2025correlated}) since model training typically prioritizes higher accuracy for the final model.}
among all factorizations, as shown in the theorem below.

\begin{restatable}
{theorem}{FullBatchConstants}
\label{thm:full_batch_constants}

Let $b=1$ and $k=n$, and let $A\in\mathbb{R}^{n\times n}$ be the prefix-sum matrix. Define the diagonal factorization $C_{\mathrm{diag}}=\operatorname{diag}\!\left((n-j+1)^{1/4}\right)_{j=1}^n$ and $B_{\mathrm{diag}}=A C_{\mathrm{diag}}^{-1}$. Then, for any  factorization $A=BC$, the following lower bound holds:
\vspace{-0.2cm}
\begin{equation}
\mathrm{RMSE}(B,C)
\ge
\mathrm{RMSE}(B_{\mathrm{diag}},C_{\mathrm{diag}})
=
\frac{1}{\sqrt{n}}
\sum_{j=1}^n \sqrt{j}
\sim
\frac{2n}{3}.
\end{equation}
Moreover, if the matrix $C$ is lower triangular with positive diagonal, then equality holds if and only if $C=\alpha C_{\mathrm{diag}}$ for some $\alpha>0$.
\end{restatable}

While DP-SGD is not constant optimal under the RMSE metric across all factorizations, the following result shows that it is also RMSE-optimal within the class of \method factorizations.

\begin{restatable}{corollary}{trivialfactorizationcorollary}
The trivial factorization $B = A$, $C = I$ is optimal for \method under the RMSE
and MaxSE metrics for $b = 1$ and $k = n$.
\end{restatable}

We show that in the limit of having a full batch so $b=1$, DP-SGD ($\lambda = 0$) is provably optimal, suggesting that as the batch size increases, the optimal $\lambda$ decreases as observed in Figure~\ref{fig:choice_of_lambda}.

\begin{figure*}[t]
    \centering
        \begin{minipage}{.4\textwidth}
        \includegraphics[height=.8\linewidth]{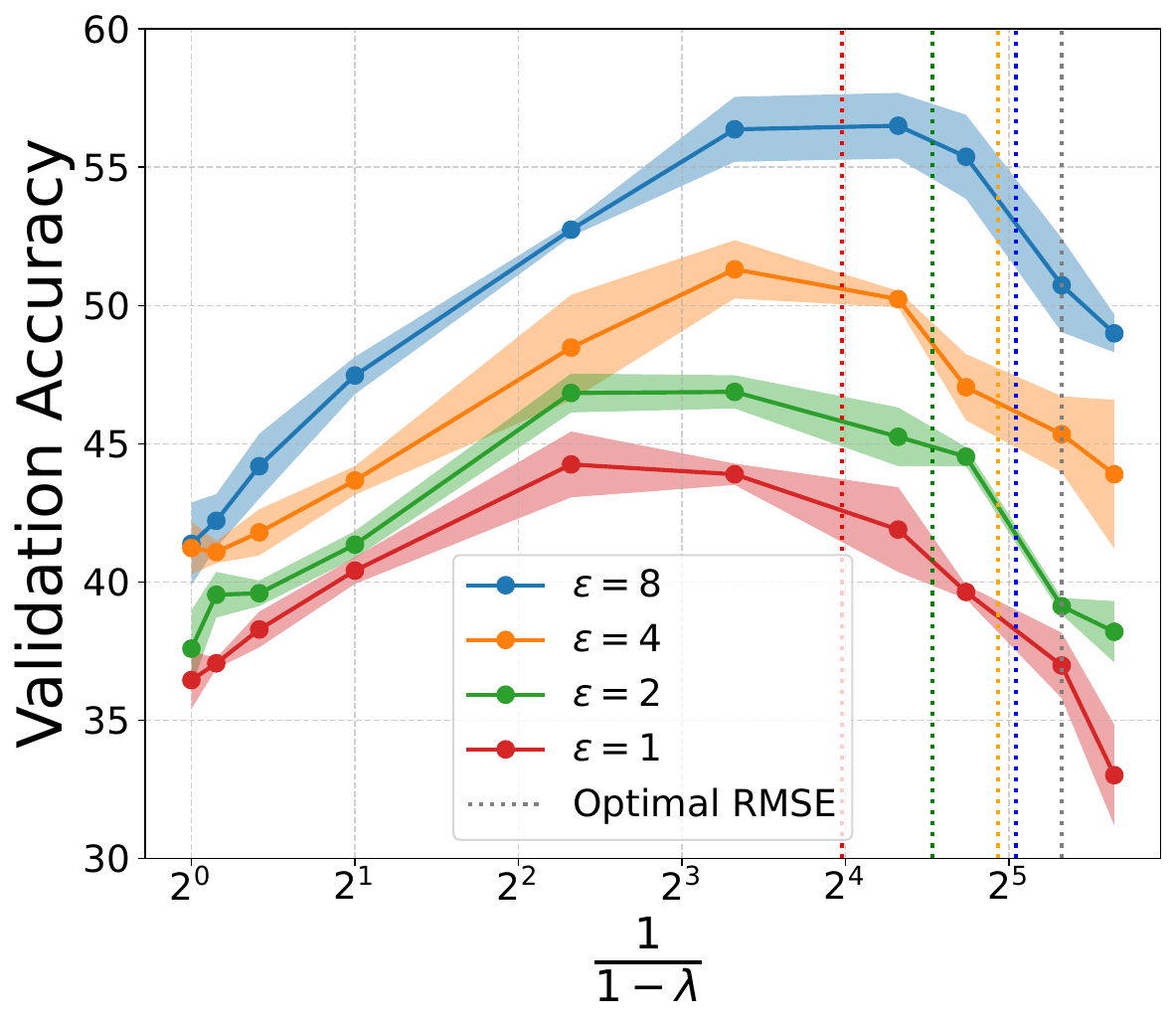}
        
        \centerline{\footnotesize CIFAR-10 $B=128$, $k=10$.}
        \end{minipage}
        \quad
        \begin{minipage}{.4\textwidth}
        \includegraphics[height=.8\linewidth]{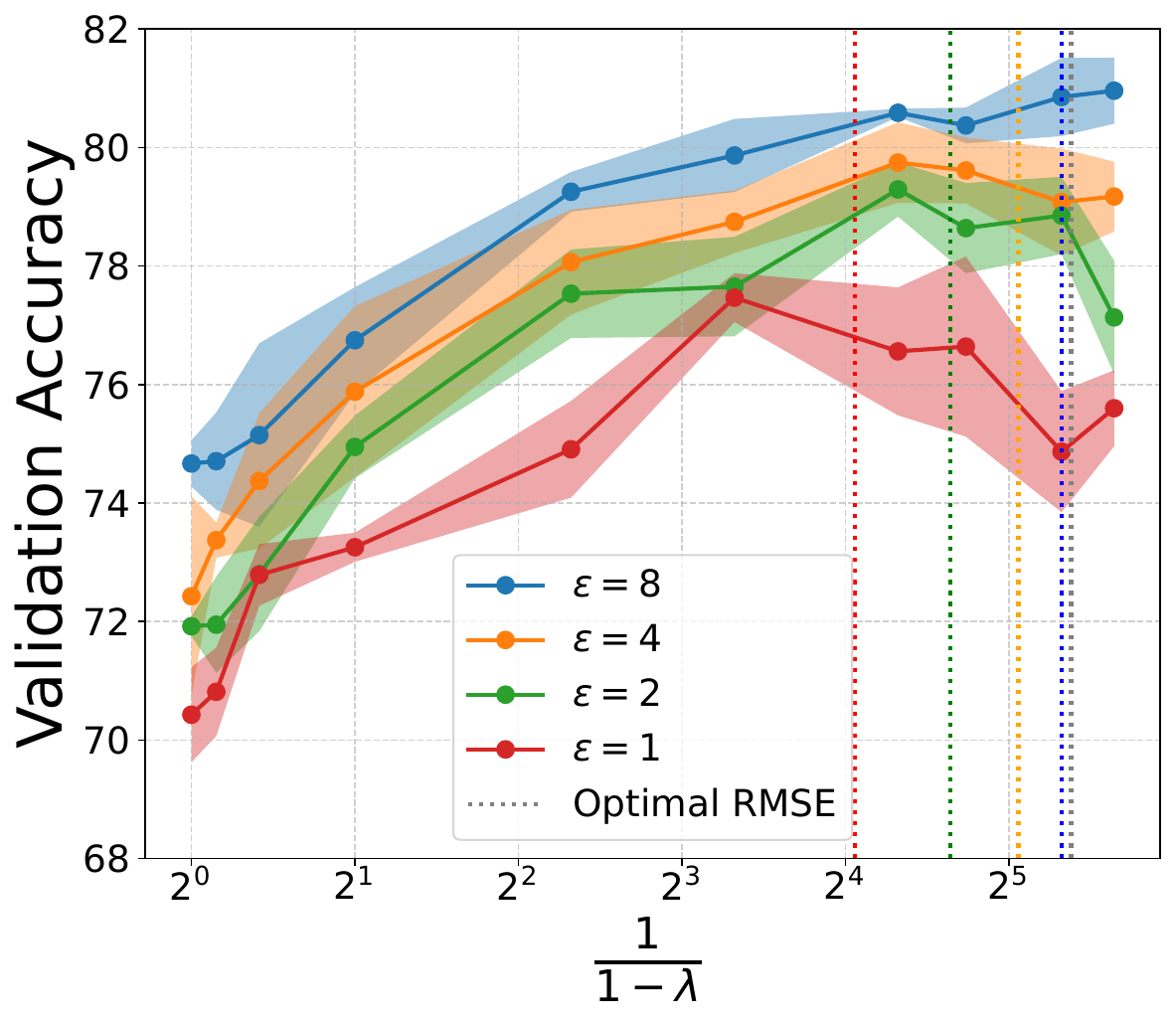}
        
        \centerline{\footnotesize IMDb $B=64$, $k=10$.}
        \end{minipage}
    \caption{Validation accuracy of \method with Balls-in-Bins amplification by subsampling for different values of $\lambda$. Vertical lines indicate the optimal values based on the amplified RMSE, suggesting that Balls-in-Bins prefers smaller values of $\lambda$. However, the value of $\lambda$ that is optimal for accuracy is even lower, at least for CIFAR-10 training.}
    \label{fig:amplification_results}
\end{figure*}

\subsection{Amplification}
\label{sub:amplification}

The most common amplification scheme for DP-SGD is Poisson subsampling \citep{abadi2016deep, koskela2020computing, zhu2022optimal}, in which each batch is formed independently and each element is included with a probability determined by the expected batch size. This method provides the strongest known privacy guarantees among sampling schemes. However, it has notable limitations. In expectation, nearly one third of the data (specifically, a fraction $1/e$) does not participate in each epoch. Moreover, Poisson subsampling is difficult to implement efficiently at scale, since forming each batch would require scanning the entire dataset and independently deciding whether to include each element. This leads to the common but undesirable practice of training with shuffling while reporting privacy guarantees under Poisson subsampling \citep{lebeda2025avoiding, beltran2024towards}.

To address these issues, an alternative sampling scheme was proposed, known as Balls in Bins (BnB) \citep{chua2024balls, choquette2024near, schuchardt2026sampling} or random allocation \citep{feldman2025privacy, feldman2026efficient}. In this approach, one first specifies the number of batches in an epoch and then assigns each data point to a random batch. This guarantees that every sample participates once per epoch, and by preprocessing the dataset, one can efficiently construct batches for multi epoch training. 

Balls-in-Bins amplification by subsampling can also be applied to matrix factorization \citep{choquette2024near} for any correlation matrix. We use it for \method, as it is also the only currently available amplification scheme that is applicable to banded inverse matrices. 
\citet{choquette2024near} compute the privacy level using Monte Carlo sampling. This procedure is effective in practice but acts as a black box and does not allow us to argue theoretically about its properties. We therefore study the effect of amplification numerically. Our experiments show that amplification by Balls in Bins subsampling prefers smaller values of $\lambda$ than those that would be chosen based solely on RMSE. In particular, the values that minimize the amplification-aware RMSE: $\tfrac{1}{\sqrt{n}}\|B\|_F \cdot \sigma_{\epsilon, \delta}(C)$ are closer to the optimal factorization for \method than those obtained using RMSE alone. See Figure~\ref{fig:amplification_results} for an illustration.

\paragraph{Practical considerations for the choice of $\lambda$.}
Our main practical observation is that, in our experiments, the optimal value of
$\lambda$ is consistently smaller than the value suggested by amplified RMSE. The
latter is itself smaller than the value suggested by RMSE, which is provably smaller
than the value suggested by MaxSE. Thus, a practical guideline is to choose a value of
$\lambda$ no larger than the one minimizing amplified RMSE: in our experiments, the
optimal choice is typically $2$--$4\times$ smaller, measured on the scale
$\frac{1}{1-\lambda}$, than the value suggested by amplified RMSE.

We also note that, when optimizing over $\lambda$ is computationally expensive, simply
using the value that minimizes amplified RMSE already yields most of the benefit in the
amplified setting.

\begin{figure}[t!]
    \centering

    \begin{subfigure}[t]{0.4\linewidth}
        \centering
        \includegraphics[width=\linewidth]{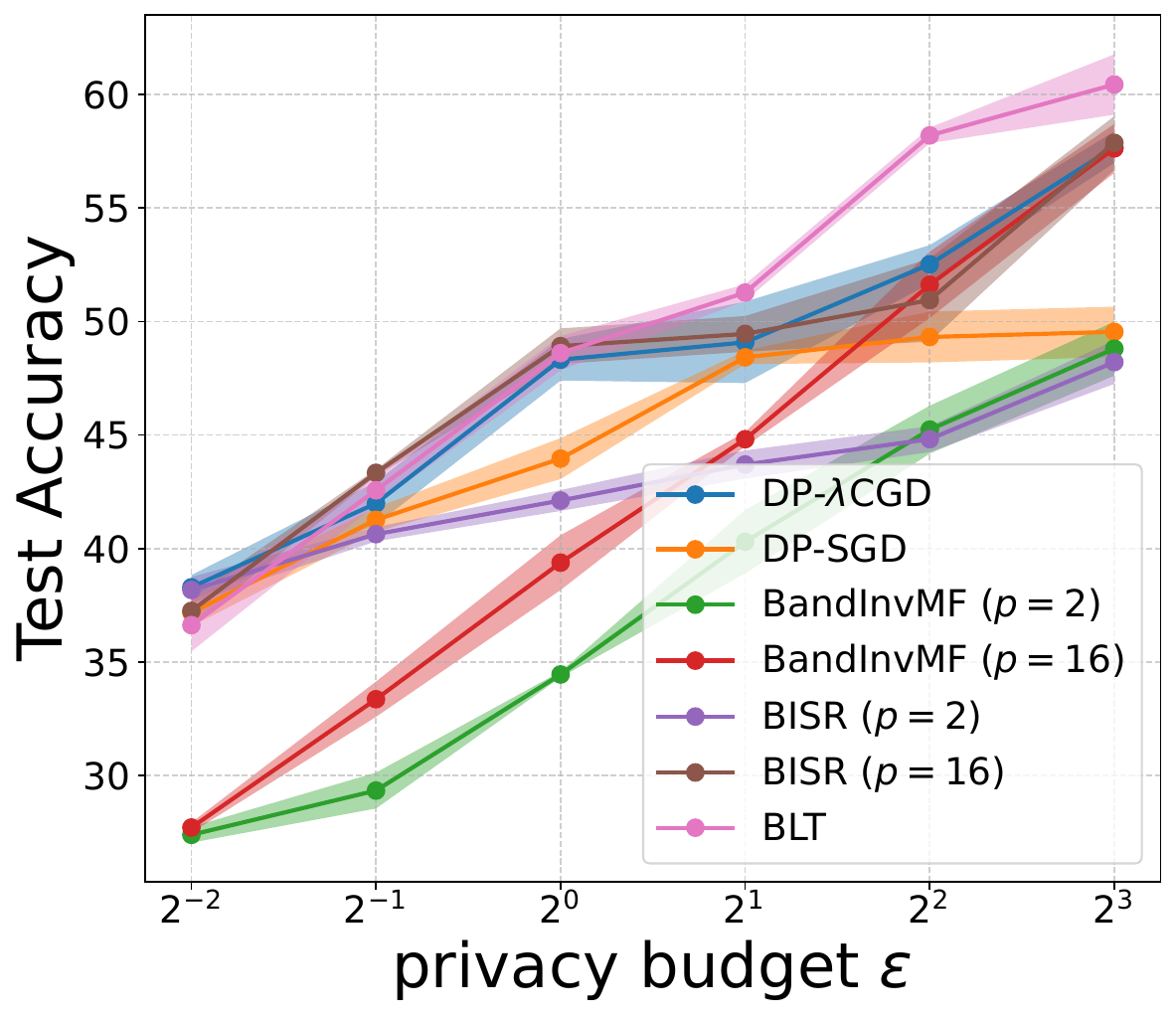}
        \caption{With amplification by subsampling.}
        \label{fig:method_comparison_cifar_10_amplification}
    \end{subfigure}
    \hspace{2cm}
    \begin{subfigure}[t]{0.4\linewidth}
        \centering
        \includegraphics[width=\linewidth]{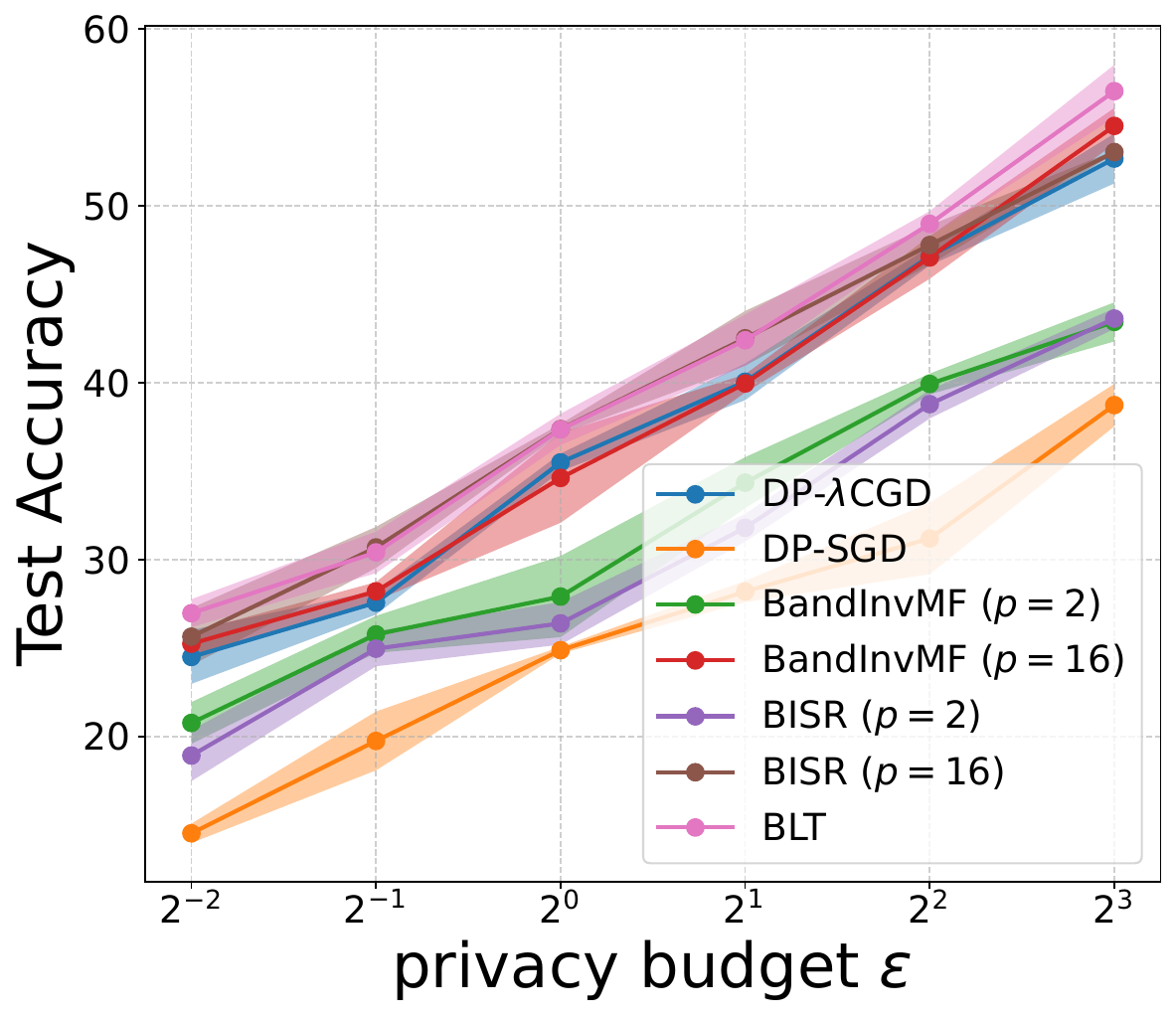}
        \caption{Without amplification by subsampling.}
        \label{fig:method_comparison_cifar_10_no_amplification}
    \end{subfigure}

    \caption{Test accuracy on CIFAR-10 across privacy budgets
($B=128$, $k=10$; error bars over three runs), comparing DP-SGD with
\method and other memory-efficient matrix-factorization baselines.
Hyperparameters for training with and without amplification by subsampling are
reported in Tables~\ref{tab:hyperparams_amplification}
and~\ref{tab:hyperparams_no_amplification}, respectively.}
    \label{fig:method_comparison_cifar_10}
\end{figure}

\section{Experiments}
We compare \method, amplified via Balls-in-Bins subsampling, to other memory-efficient alternatives, including DP-SGD with Poisson subsampling (a commonly used amplification that empirically yields a slightly lower noise multiplier) and the banded inverse factorizations Banded Inverse Matrix Factorization (BandInvMF) and Banded Inverse Square Root (BISR) \citep{kalinin2025back} with bandwidths $2,4,16$.
 We also include Buffered Linear Toeplitz (BLT)  factorization \citep{dvijotham2024efficient}; although BLT cannot be run efficiently without additional memory, we find it to be a competitive low-memory approach, requiring to store a small buffer of $4-5$ vectors. In Table~\ref{tab:eps_all}, we compare the methods using the RMSE metric, with and without amplification. In the table, we also report banded factorizations such as Banded Matrix Factorization (BandMF) \citep{scalingmckenna2024} and Banded Square Root (BSR) \citep{kalinin2024}; however, they do not perform competitively at small bandwidths ($2$, $4$, or even $16$).

Figure~\ref{fig:method_comparison_cifar_10} reports CNN test accuracy on CIFAR-10, both with and without privacy amplification (see Tables~\ref{tab:hyperparams_amplification} and \ref{tab:hyperparams_no_amplification} for hyperparameters). In the amplified setting, we use Poisson subsampling for DP-SGD, which provides stronger amplification than BnB and therefore makes DP-SGD a more competitive baseline.
\method outperforms DP-SGD across the evaluated privacy budgets, including $\epsilon=1/4, \delta=10^{-5}$, and performs on par with BISR with bandwidth $16$ while using bandwidth $2$, reducing noise-regeneration work by $8\times$. BLT performs better in parts of the low-privacy regime, but requires either storing a noise buffer or regenerating the full sequence. Overall, \method provides a strong accuracy--efficiency trade-off without storing past noise vectors. 
We also provide large scale experiments on full model fine-tuning of BERT-Base on IMDb and fine-tuning two blocks of ViT-Base on CIFAR-100. The results are reported in Figure~\ref{fig:test_accuracy_large}.

\section*{Discussion}
We proposed \method, a time and memory efficient algorithm for introducing noise
correlation in differentially private model training. Our experiments show that \method
improves upon the utility of DP-SGD and is competitive with other memory-efficient baselines,
while running at nearly the same speed as DP-SGD. From the perspective of noise correlation,
\method is a particularly simple nontrivial factorization, making it a useful setting for
isolating phenomena that are difficult to study in more general classes of factorizations.
In particular, prior work has observed that the RMSE-optimal factorization does not
necessarily yield the best downstream performance. The single parameter structure of \method
allows us to study this mismatch directly by interpolating between DP-SGD and the
RMSE-optimal factorization. A promising direction for future work is to develop a similar
interpolation for larger bandwidths, which could have smaller RMSE while
preserving the time and memory efficiency of noise regeneration.

\subsection*{AI use statement}

In this work, we used generative AI tools for code assistance, for language editing, and to help draft the proof of Theorem~\ref{thm:full_batch_constants}.
We did not use generative AI tools for experiment design, for proofs of theoretical statements other than Theorem~\ref{thm:full_batch_constants}, or for developing research ideas. We have reviewed all AI-assisted work. In particular, all proofs were  manually written and verified by the authors, and all AI-assisted code was carefully reviewed by the authors.

We take responsibility for the final content of this work, including any text, claims, or artifacts produced with the aid of generative AI.

\section*{Acknowledgment}
We thank Jalaj Upadhyay for his valuable feedback on the early version of the paper.

Nikita Kalinin: This work is
supported in part by the Austrian Science Fund (FWF) [10.55776/COE12]. 

Rasmus Pagh was supported by a Data Science Distinguished Investigator grant from Novo Nordisk Fonden.

\bibliography{lit}
\bibliographystyle{iclr2027_conference}

\appendix
\clearpage

\section{Additional Experiments}

\begin{figure*}[h!]
    \centering
    \begin{subfigure}{0.30\linewidth}
        \centering
        \includegraphics[width=\linewidth]{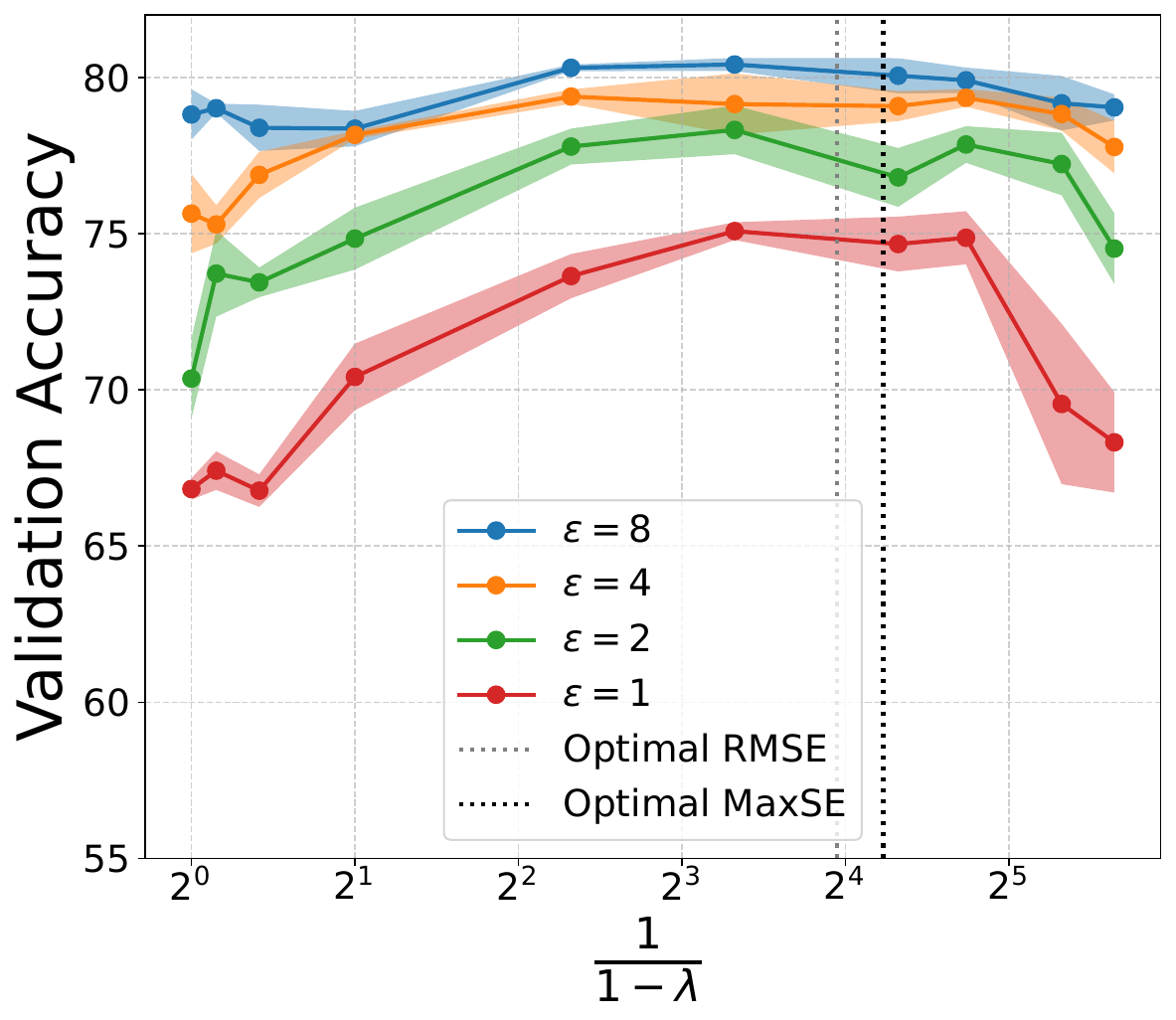}
        \caption{IMDB:  $k=20$}
    \end{subfigure}
    \begin{subfigure}{0.30\linewidth}
        \centering
        \includegraphics[width=\linewidth]{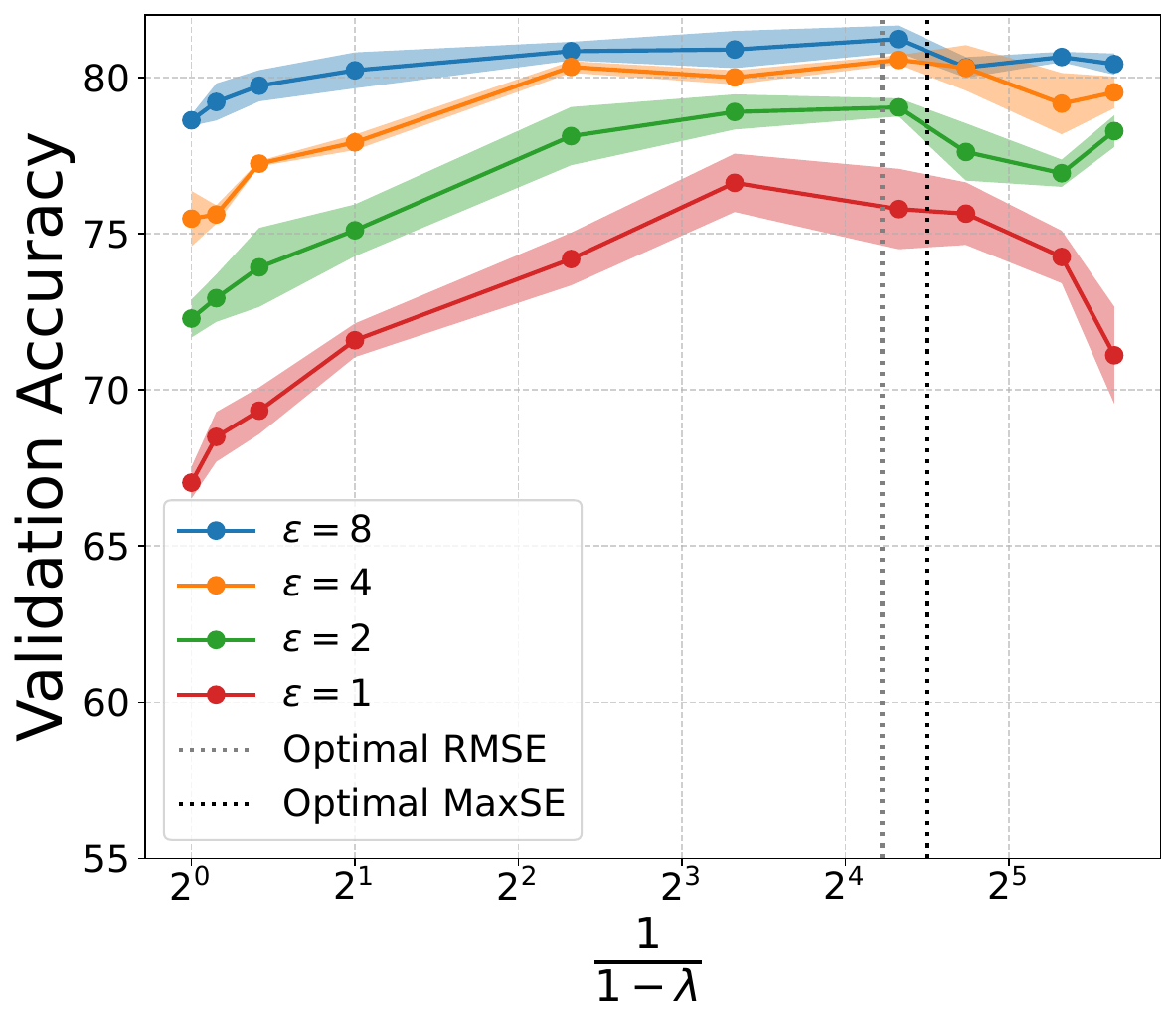}
        \caption{IMDB:  $k=40$}
    \end{subfigure}
    \begin{subfigure}{0.30\linewidth}
        \centering
        \includegraphics[width=\linewidth]{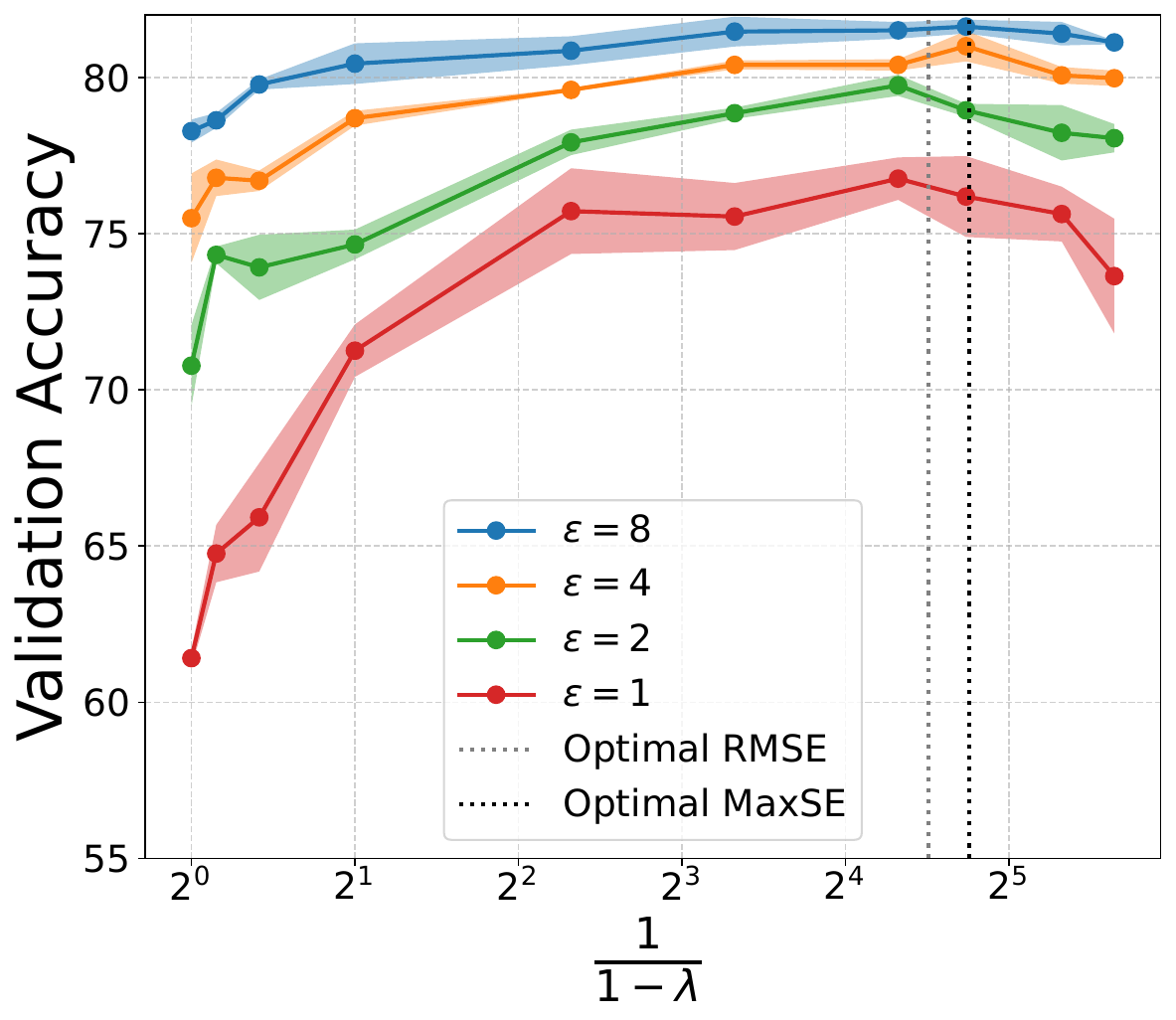}
        \caption{IMDB:  $k=80$}
    \end{subfigure}

    \begin{subfigure}{0.30\linewidth}
        \centering
        \includegraphics[width=\linewidth]{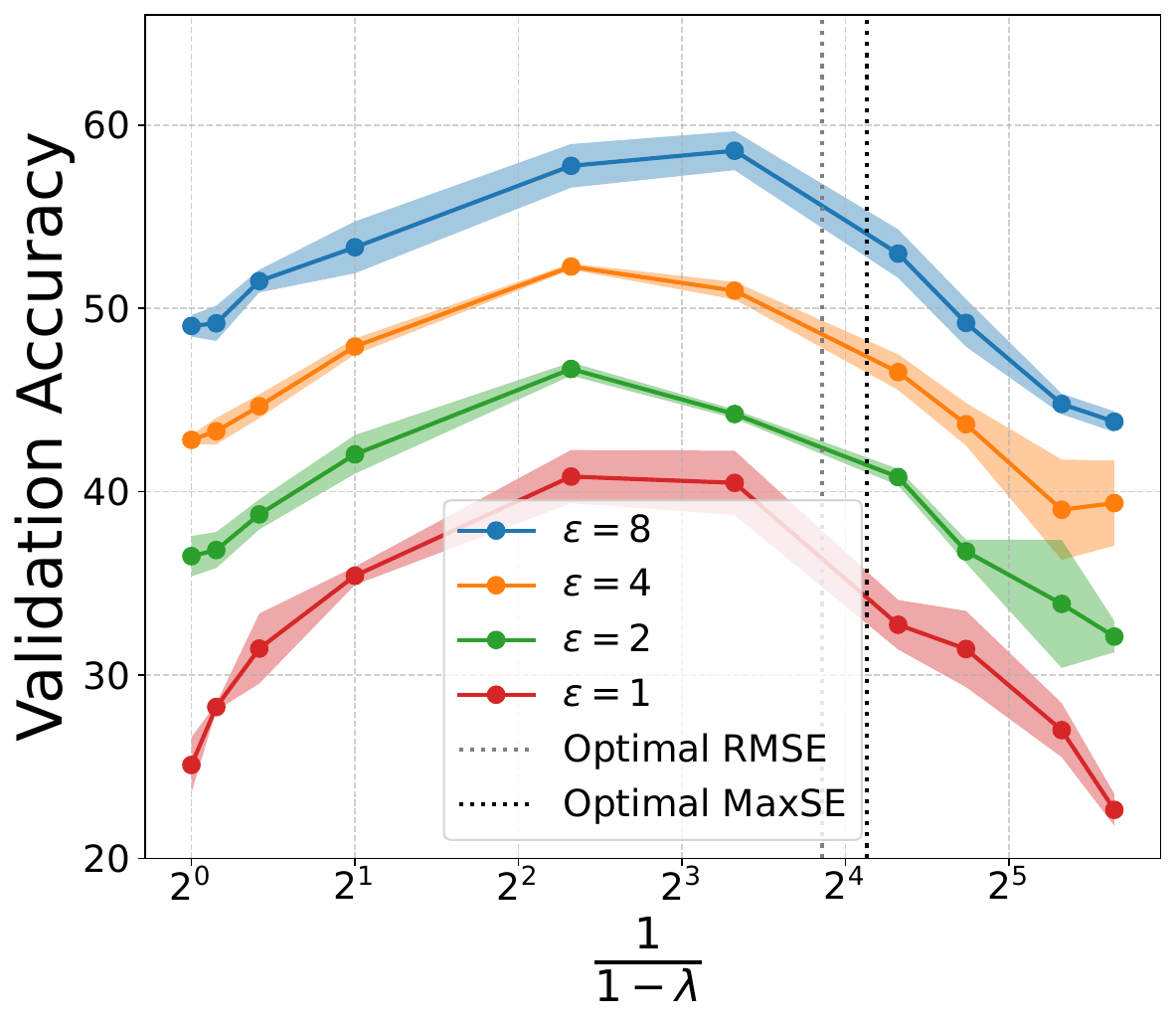}
        \caption{CIFAR-10:  $k=20$}
        \label{fig:imdb128} 
    \end{subfigure}
    \begin{subfigure}{0.30\linewidth}
        \centering
        \includegraphics[width=\linewidth]{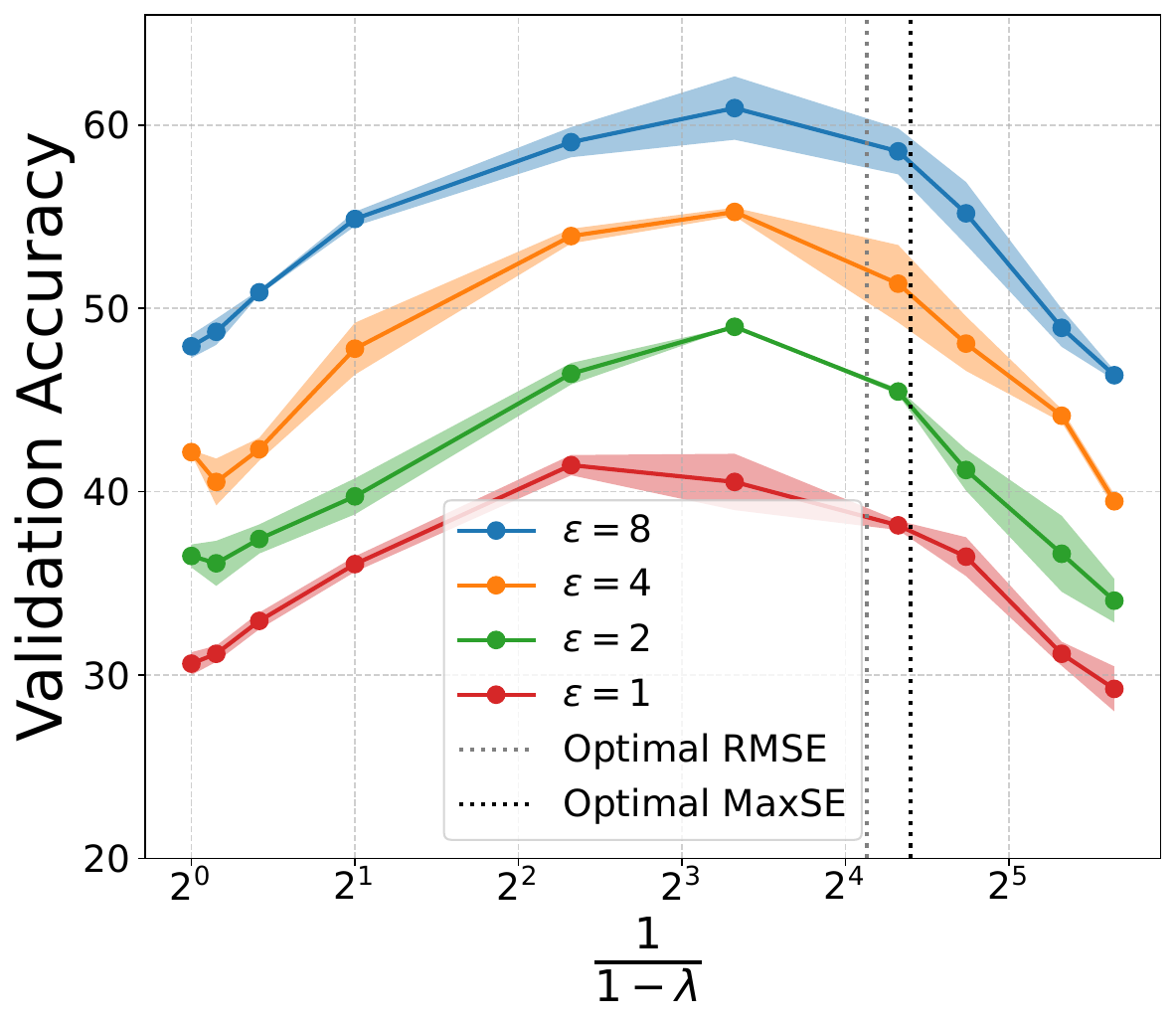}
        \caption{CIFAR-10: $k=40$}
    \end{subfigure}
    \begin{subfigure}{0.30\linewidth}
        \centering
        \includegraphics[width=\linewidth]{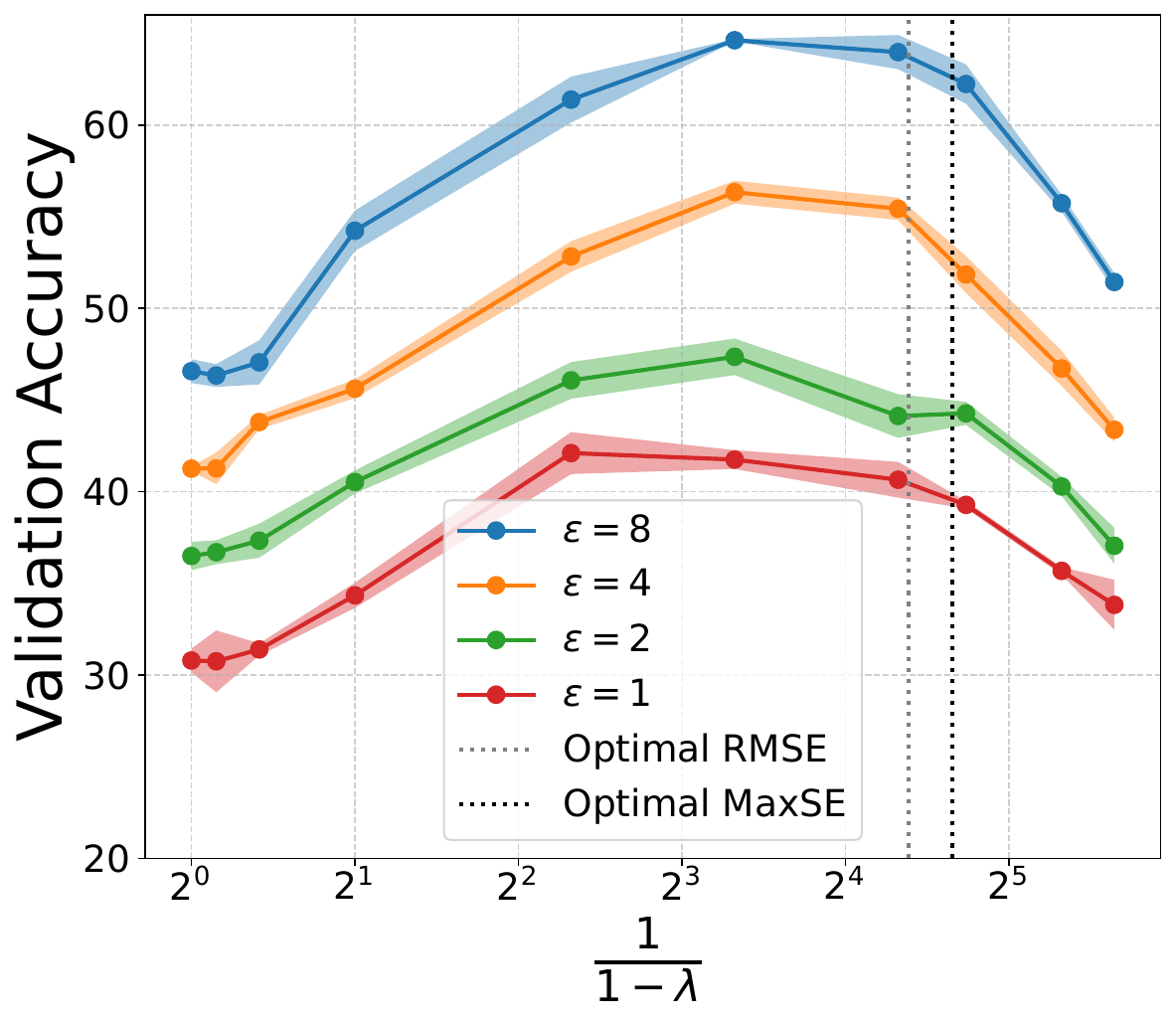}
        \caption{CIFAR-10: $k=80$}
    \end{subfigure}

    \caption{Validation accuracy with error bars based on three runs for fine-tuning BERT-tiny on the IMDB dataset and CNN on CIFAR-10.  In these runs, we fix the batch size $B$ to be $512$ for IMDb and $1024$ for CIFAR-10, we also fix the number of iterations per epoch, $b = |D|/B$, and vary the number of epochs $k$. We observe improved utility compared to Figure~\ref{fig:choice_of_lambda}, suggesting that using a larger batch size and training for longer is preferable for overall accuracy. We also find that, in this regime, the optimal mechanism is not DP-SGD. In fact, the trend is the opposite: as the number of epochs increases, the optimal $\lambda$ also increases. This suggests that one can benefit from \method even in long training runs, which are common for DP-SGD. 
   }
    \label{fig:experiments_fixed_batch_size}
\end{figure*}

\begin{figure}[h!]
    \centering

    \begin{subfigure}[t]{0.4\linewidth}
        \centering
        \includegraphics[width=\linewidth]{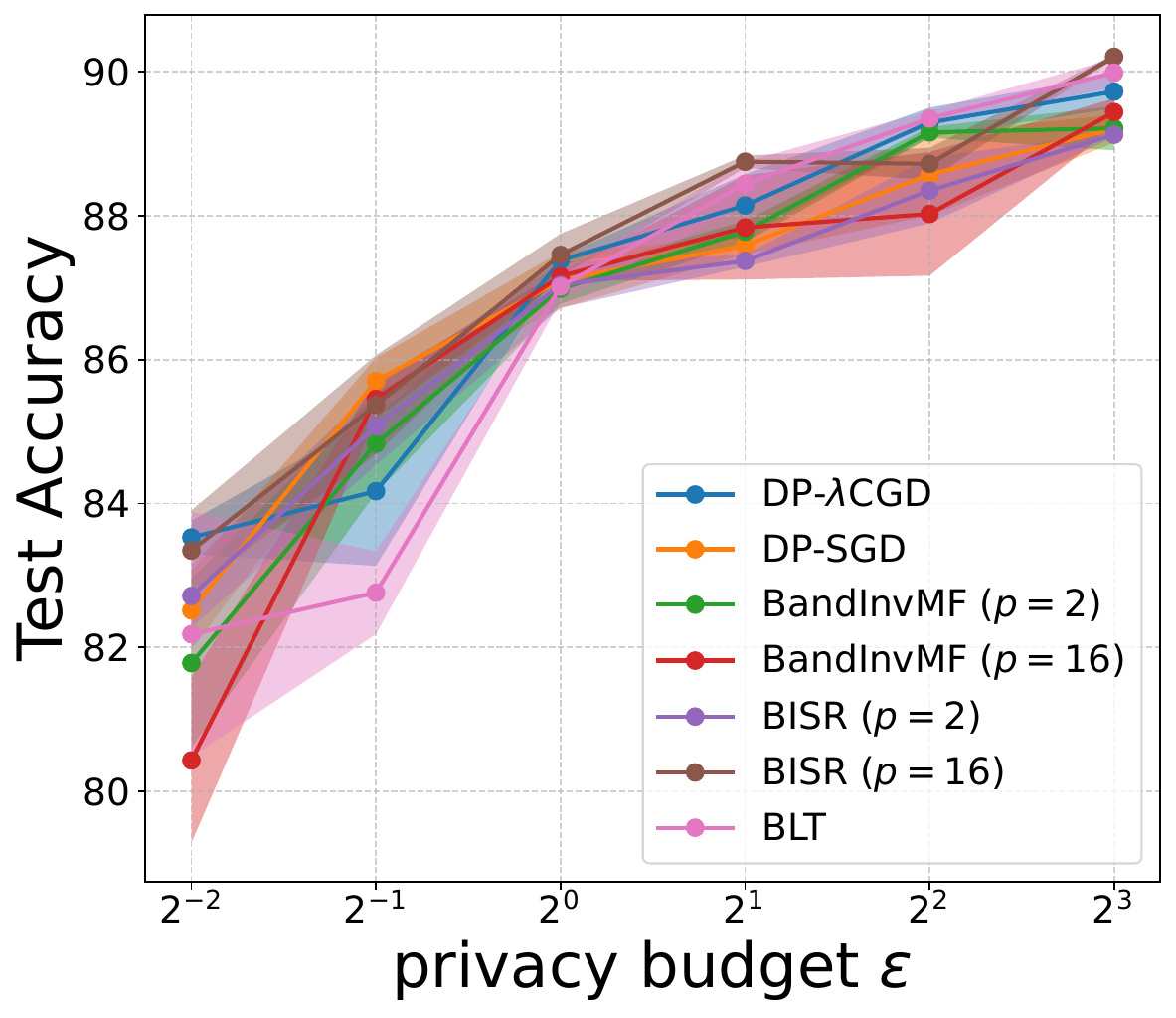}
        \caption{BERT-base on IMDb.}
        \label{fig:test_accuracy_imdb}
    \end{subfigure}
    \hspace{1cm}
    \begin{subfigure}[t]{0.4\linewidth}
        \centering
        \includegraphics[width=\linewidth]{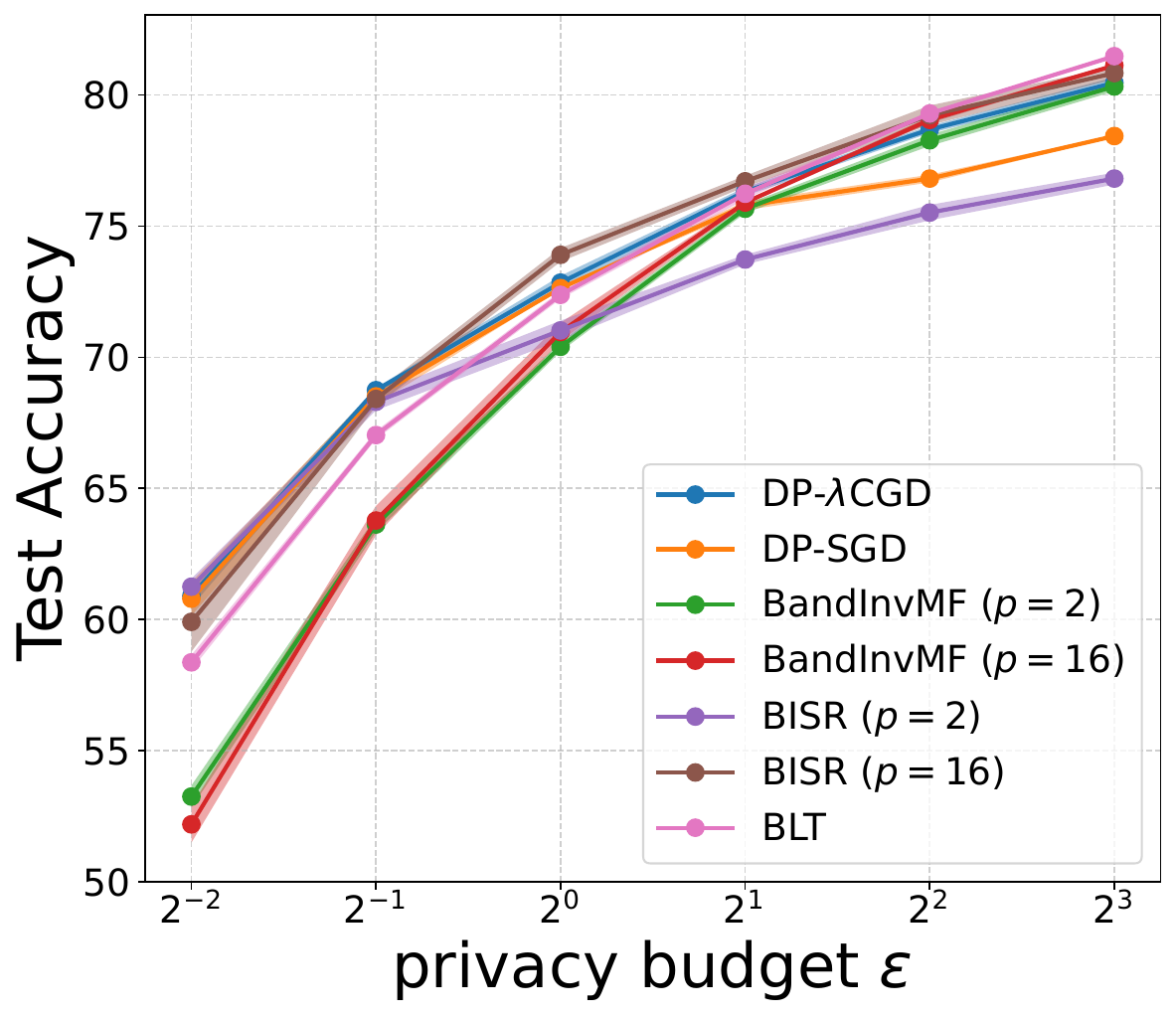}
        \caption{ViT-base on CIFAR-100.}
        \label{fig:test_accuracy_cifar100}
    \end{subfigure}

    \caption{
        Test accuracy as a function of the privacy budget $\epsilon$ for
        \textbf{(a)} BERT-base fine-tuned on IMDb for sentiment analysis and
        \textbf{(b)} ViT-base fine-tuned (two blocks) on CIFAR-100 for image classification.
        Both models are trained with batch size $512$ for $10$ epochs.
        Our proposed \method outperforms DP-SGD at most privacy levels, achieves
        strong performance under very strict privacy ($\epsilon=1/4$), and is
        competitive with BLT, BandInvMF and BISR factorizations.
        Error bars are computed from three runs.
    }
    \label{fig:test_accuracy_large}
\end{figure}

\begin{table}[h]
\caption{Time to generate and correlate the noise at the final ($n$th) iteration and buffer size to produce correlated noise in dimension $d$. For BLT, BSR, and BandMF, regeneration requires replaying from iteration 1, so the last-iteration cost is $nd$. All methods also require correlating noise by summing $p$ vectors of size $d$ ($+pd$ in the runtime).}
\centering
\small
\setlength{\tabcolsep}{4.5pt}
\renewcommand{\arraystretch}{1.15}
\begin{tabular}{lcccc}
\toprule
& \multicolumn{2}{c}{Time per iter. Gen + Corr} & \multicolumn{2}{c}{Buffer size} \\
\cmidrule(lr){2-3}\cmidrule(lr){4-5}
Method & w/o regen. & w/ regen. & w/o regen. & w/ regen. \\
\midrule
DP-SGD     & $d + d$  & $d + d$  & --        & -- \\
\method    & $d + 2d$  & $2d + 2d$ & $d$       & -- \\
BandInvMF  & $d + pd$  & $pd+ pd$ & $(p-1)d$  & -- \\
BISR       & $d+ pd$  & $pd+ pd$ & $(p-1)d$  & -- \\
BLT        & $d+ pd$  & $\textcolor{red}{n}d+ pd$ & $(p-1)d$  & -- \\
BSR        & $d+ pd$  & $\textcolor{red}{n}d+ pd$ & $(p-1)d$  & -- \\
BandMF     & $d+ pd$  & $\textcolor{red}{n}d+ pd$ & $(p-1)d$  & -- \\
\bottomrule
\end{tabular}
\label{tab:noise_time_buffer}
\end{table}

\begin{table}[h]
\centering
\caption{RMSE for different mechanisms across privacy budgets $\epsilon$ in the setting of training on CIFAR-10 with batch size $128$ over $k=10$ epochs. The first column reports results without amplification by subsampling; in this setting, BandMF \citep{scalingmckenna2024} with full bandwidth achieves the best performance. When Balls-in-Bins amplification via subsampling is applied, other methods perform better, including \method, BLT \citep{dvijotham2024efficient}, and BISR \citep{kalinin2025back} with bandwidth $p=16$. In the high-privacy regime, subsampling substantially affects RMSE, and methods that account for amplification, especially \method, benefit the most.}
\label{tab:eps_all}
\begin{tabular}{lccccccc}
\toprule
Method & $\epsilon=8$ (w/o Amp) & $\epsilon=8$ & $\epsilon=4$ & $\epsilon=2$ & $\epsilon=1$ & $\epsilon=0.5$ & $\epsilon=0.25$ \\
\midrule
BandInvMF ($p=2$)  & 12.69 & 9.70  & 14.90 & 24.41 & 43.66 & 82.77 & 125.87 \\
BandInvMF ($p=4$)  & 10.27 & 8.07  & 12.48 & 20.76 & 37.37 & 70.90 & 130.48 \\
BandInvMF ($p=16$) & 8.54  & 6.56  & 11.23 & 20.17 & 37.09 & 71.08 & 130.90 \\
BandInvMF ($p=64$) & 8.15  & 6.39  & 11.36 & 20.87 & 39.13 & 74.06 & 140.38 \\
BandInvMF ($p=b$)  & 7.87  & 5.66  & 9.86  & 17.76 & 33.25 & 62.84 & 119.20 \\
\addlinespace
BLT & 8.14 & 5.87 & 10.67 & 19.50 & 35.97 & 69.41 & 127.69 \\
\addlinespace
DP-$\lambda$CGD ($\lambda=0.975$) & 12.73 & 9.33  & 14.03 & 21.45 & 36.40 & 68.37 & 125.69 \\
DP-$\lambda$CGD ($\lambda=0.95$)  & 14.74 & 10.27 & 15.28 & 21.45 & 31.70 & 56.75 & 103.44 \\
DP-$\lambda$CGD ($\lambda=0.9$)   & 19.72 & 13.25 & 18.42 & 24.73 & 33.66 & 54.33 & 97.73 \\
\addlinespace
BISR ($p=2$)  & 48.45 & 30.52 & 40.58 & 49.44 & 58.16 & 71.20 & 105.09 \\
BISR ($p=4$)  & 33.47 & 21.00 & 27.97 & 34.87 & 42.50 & 58.27 & 98.52 \\
BISR ($p=16$) & 17.95 & 11.50 & 15.75 & 21.10 & 20.25 & 52.81 & 96.45 \\
BISR ($p=64$) & 10.50 & 6.81  & 9.87  & 15.92 & 28.80 & 53.89 & 101.90 \\
BISR ($p=b$)  & 8.45  & 6.34  & 11.30 & 20.77 & 38.96 & 73.70 & 139.68 \\
\addlinespace
BSR ($p=2$)  & 62.51 & 39.65 & 51.87 & 62.97 & 74.14 & 88.33 & 116.43 \\
BSR ($p=4$)  & 46.80 & 29.29 & 39.32 & 48.13 & 57.03 & 70.47 & 104.90 \\
BSR ($p=16$) & 26.27 & 16.88 & 22.41 & 28.59 & 36.78 & 55.03 & 97.09 \\
BSR ($p=64$) & 14.89 & 9.62  & 13.32 & 18.88 & 29.34 & 52.16 & 97.51 \\
BSR ($p=b$)  & 8.15  & 5.49  & 9.31  & 17.06 & 31.95 & 60.43 & 114.52 \\
\addlinespace
BandMF ($p=2$)  & 59.44 & 37.67 & 49.08 & 59.95 & 71.17 & 86.25 & 114.94 \\
BandMF ($p=4$)  & 42.29 & 27.00 & 36.99 & 45.90 & 55.35 & 68.64 & 104.42 \\
BandMF ($p=16$) & 22.05 & 15.00 & 21.20 & 28.13 & 37.06 & 56.29 & 99.33 \\
BandMF ($p=64$) & 12.58 & 8.77  & 13.25 & 19.76 & 31.40 & 56.07 & 104.68 \\
BandMF ($p=b$)  & 7.77  & 5.84  & 10.42 & 19.16 & 35.94 & 68.02 & 128.79 \\
\addlinespace
DP-SGD & 83.85 & 21.82 & 26.27 & 31.68 & 40.10 & 59.17 & 100.27 \\
\bottomrule
\end{tabular}%
\end{table}

\clearpage
\section{\method Algorithm}

\begin{figure}[h]
    \centering
    \includegraphics[width=0.5\linewidth]{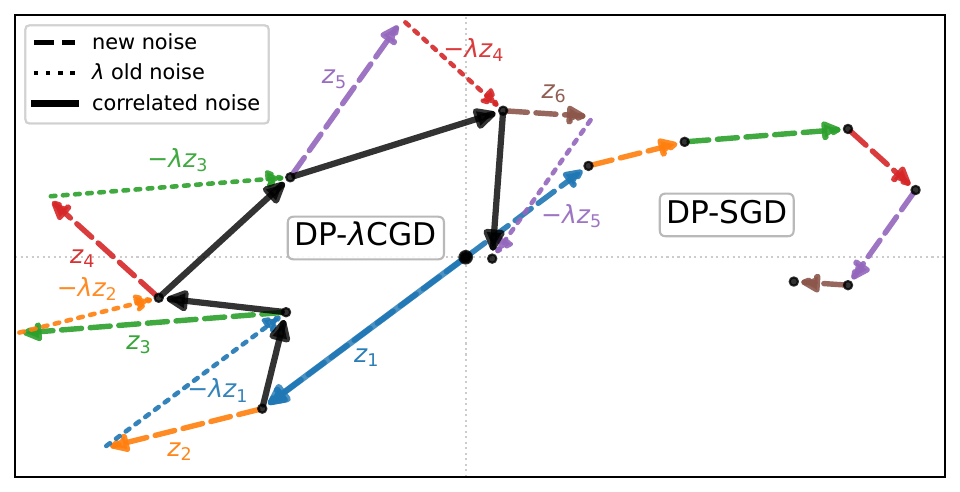}
    \caption{Visualization of noise correlation in \method. At each step, fresh noise $z_i$ is generated, and a $\lambda$-fraction of the previously added noise $z_{i-1}$ is canceled. In contrast, DP-SGD adds independent noise at each iteration. Although the per-iteration noise vectors in \method have larger variance, their correlation leads to partial noise cancellation, yielding a lower total variance. 
    }
    \label{fig:visualisation}
\end{figure}

\begin{algorithm}[h]
\caption{\method}
\label{alg:method}
\begin{algorithmic}[1]
\Require Initialization $\theta_0 \in \mathbb{R}^d$, dataset $\mathcal{D}$, batch size $B$, clip norm $\zeta$, learning rate $\eta>0$, loss $\ell(\theta,d)$, parameter $\lambda \in [0,1)$, number of epochs $k$, target privacy level $\epsilon,\delta$, number of iterations $n$, pseudorandom noise generator $\mathrm{gen}$.
\State Construct matrix $C_{\lambda} = \mathrm{Lower\_Triangular\_Toeplitz}(1, \lambda, \dots, \lambda^{n-1})$
\If{use Balls in Bins accountant}
    \State Compute noise multiplier $\sigma = \mathrm{BnB}(C_{\lambda}, k, \epsilon, \delta)$ and allocate batches $D_0, \dots, D_{|\mathcal{D}|/B - 1}$
\ElsIf{no amplification}
    \State $\sigma = \mathrm{sens}_{k,|\mathcal{D}|/B}(C_{\lambda})\,\sigma_{\epsilon,\delta}$, where $\sigma_{\epsilon,\delta}$ is from the Gaussian mechanism
\EndIf
\State $Z_0 \xleftarrow{} 0$
\For{$i=1$ to $n$}
    \If{use Balls in Bins subsampling}
        \State $S_i = D_{i - 1 \bmod (n/k)}$ \Comment{Sample batch}
    \ElsIf{no amplification}
        \State $S_i = \{d_{(i-1)B \bmod |\mathcal{D}|},\dots,d_{(iB-1) \bmod |\mathcal{D}|}\}$
    \EndIf
    \For{$j=1$ to  $|S_i|$}
        \State $g_j \gets \nabla_\theta \ell(\theta_{i-1}, d_j)$
        \State $\tilde g_j \gets \min\bigl(1, \frac{\zeta}{\|g_j\|_2}\bigr)\, g_j$ \quad \Comment{per-example clipping}
    \EndFor
    \State $x_i \gets \sum_{j=1}^{|S_i|} \tilde g_j$ \qquad \Comment{aggregate clipped gradients}
    \If{$i \ne 1$}
    \State $\mathrm{gen} \gets \mathrm{gen.set\_state(\mathrm{GSt})}$
    \State $Z_{i-1} \sim \mathrm{gen}()$ \qquad \Comment{regenerate noise}
    \EndIf
    \State $\hat x_i \gets x_i - \zeta\sigma \lambda Z_{i-1}$
    \State $\mathrm{GSt} \gets \mathrm{gen.get\_state()}$
    \State $Z_i \sim \mathrm{gen}()$ \qquad \Comment{generate fresh noise}
    \State $\hat x_i \gets \hat x_i + \zeta\sigma Z_i$
    \State $\theta_i \gets \theta_{i-1} - \frac{\eta}{B}\hat x_i$ \qquad \Comment{model update}
\EndFor
\State \textbf{return} $\theta_n$
\end{algorithmic}
\end{algorithm}

We present the full algorithm for \method in Algorithm~\ref{alg:method}. In addition to noise correlation, we emphasize the role of the subsampling accountant $\mathrm{BnB}$ and the batch sampling scheme. The noise regeneration technique assumes a pseudo-random generator $\mathrm{gen}$ that supports setting and retrieving its state via $\mathrm{gen.set\_state()}$ and $\mathrm{gen.get\_state()}$, respectively. We illustrate the noise correlation process in Figure~\ref{fig:visualisation}.

\clearpage

\section{Column Normalized \method}

\begin{figure}
    \centering
    \includegraphics[width=0.4\linewidth]{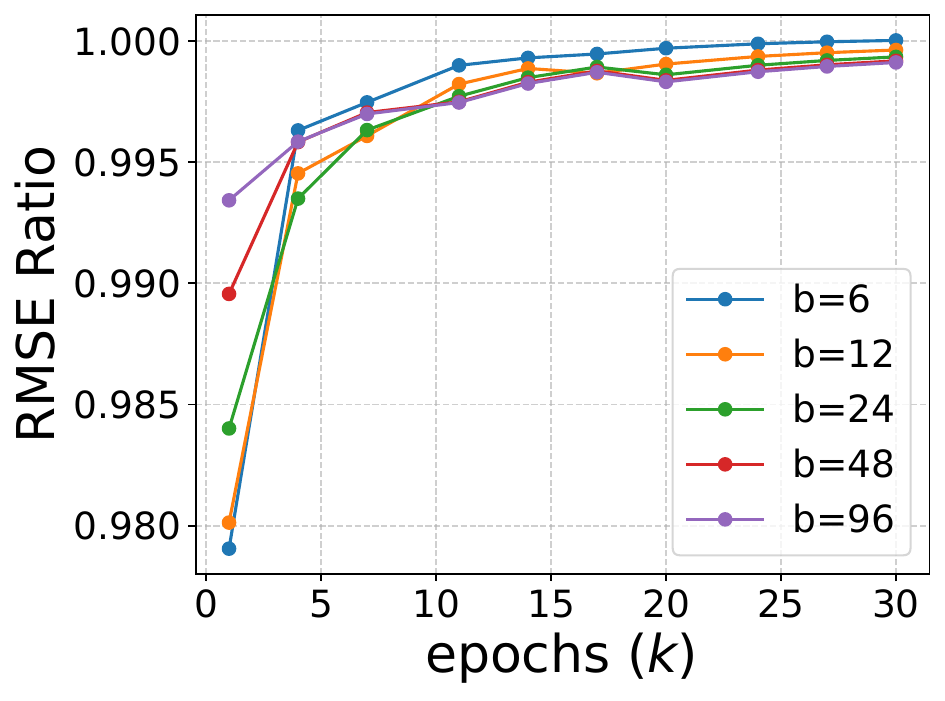}
    \caption{Ratio of RMSE for column-normalized \method to \method for varying number of epochs $k$ and number of iterations per epoch (separation parameter) $b$. Small $k$  yields a larger improvement from normalization.}
    \label{fig:rmse_normalized_ratio}
\end{figure}
The performance of \method can be further improved by column normalization.
Let
$D = \mathrm{diag}(d_1, \dots, d_n)$
denote the column norms of the matrix $C_\lambda$, and consider the normalized
\method factorization with strategy matrix $C_\lambda D^{-1}$.

It has been observed that column normalization improves performance for banded
matrices \citep{pillutla2025correlated}, and it provably improves both RMSE and
MaxSE for the square-root factorization under single participation
\citep{henzinger2025normalized}. The efficiency of the \method is preserved; in fact any linear noise correlation via matrix $C^{-1}$ can be done without loss of efficiency, we formulate it as the following proposition.

\begin{proposition}[Cost of column normalization]
Let $C\in\mathbb{R}^{n\times n}$ be invertible and let $Z\in\mathbb{R}^n$ be any noise vector.
Let $D=\mathrm{diag}(d_1,\dots,d_n)$ with $d_j>0$ (e.g., $d_j=\|C_{:,j}\|_2$), and define the
column-normalized matrix $\widetilde C \;=\; C D^{-1}$.
Then for any $Z$, $
\widetilde C^{-1} Z \;=\; D\,(C^{-1}Z)$.
Consequently, if computing $C^{-1}Z$ has time complexity $T(n)$, then computing
$\widetilde C^{-1}Z$ has time complexity $T(n)+\mathcal{O}(n)$, i.e., only $\mathcal{O}(1)$
additional work per component (one scaling by $d_i$).
\end{proposition}

\begin{proof}
Since $\widetilde C = C D^{-1}$, we have $\widetilde C^{-1} = D C^{-1}$. Multiplying by $Z$
gives $\widetilde C^{-1}Z = D(C^{-1}Z)$. The extra cost beyond $C^{-1}Z$ is applying the
diagonal matrix $D$, which is $n$ scalar multiplications.
\end{proof}

All efficiency properties are preserved: to correlate the noise $Z$, we first
multiply it by $C_\lambda^{-1}$, which we have shown can be computed
efficiently, and then rescale it by the diagonal matrix $D$.

There is, however, one subtle issue. The matrix $C_\lambda D^{-1}$ is no longer
Toeplitz, and therefore the sensitivity Theorem~\ref{thm:sensitivity} cannot be
applied directly. In this work, we generalize the sensitivity theorem specifically
for the normalized \method in the following theorem.

\begin{restatable}{lemma}{NormalizedSensitivity}
\label{lem:normalized-sensitivity}
The sensitivity of the column-normalized \method with strategy matrix $C_\lambda$ and column norms
$D = \mathrm{diag}(d_1,\dots,d_n)$ is given by
\begin{equation}
    \mathrm{sens}_{k,b}(C_\lambda D^{-1})
    =
    \Big\|
        \sum_{j=0}^{k-1} (C_\lambda D^{-1})_{:,\, jb+1}
    \Big\|_2 .
\end{equation}
\end{restatable}

For the amplified version, there are no such restrictions: the Balls-in-Bins accountant
applies to any positive lower-triangular matrix, which includes the (normalized)
\method matrices. 

We show numerically that the column normalization leads to the improvement in RMSE metric, see Figure \ref{fig:rmse_normalized_ratio}.

\begin{restatable}{lemma}{rmseComparisonLemma}
\label{lem:rmse-comparison}
For any $\lambda \in (0,1)$ under the single-participation model, the RMSE of column-normalized \method is smaller than that of the \method.
\end{restatable}
\section{Full Experiment Details}

\begin{table}[h]
\centering
\caption{Runtime comparison (seconds) across models and bandwidths for CIFAR-10 training over one epoch. CNN, WideResNet-40, and ViT-B/16 are trained using 8 CPU cores and an NVIDIA A100 GPU. We use a batch size of 512 with a physical batch size of 64. For ViT-L/16 we use NVIDIA H100 GPU with 64 CPU cores and physical batch size 16. The standard deviations are computed over $3$ runs for each model, except for CNN, for which we used $10$ runs to get more reliable estimates.}
\label{tab:runtime_comparison}
\resizebox{\textwidth}{!}{
\begin{tabular}{lccccccccc}
\hline
 & \multicolumn{3}{c}{\textbf{DP-SGD}} 
 & \multicolumn{3}{c}{\textbf{GPU Noise Regeneration}} 
 & \multicolumn{3}{c}{\textbf{CPU Noise Buffer}} \\
\cline{2-10}
\textbf{Model} 
 &   &   & 
 & 2 & 4 & 16 
 & 2 & 4 & 16 \\
\hline
CNN (0.3M)
 & & $18.5 \pm 0.3$ &
 & $18.8 \pm 0.2$ & $18.9 \pm 0.3$ & $19.4 \pm 0.2$
 & $19.1 \pm 0.9$ & $19.0 \pm 0.3$ & $19.5 \pm 1.5$ \\
 
WideResNet-40 (9M)
 & & $133.7 \pm 2.2$ & 
 & $133.5 \pm 1.2$ & $136.9 \pm 1.3$ & $144.1 \pm 1.0$
 & $143.5 \pm 2.3$ & $143.1 \pm 1.5$ & $149.6 \pm 1.9$ \\

ViT-B/16 (86M)
 & & $825.3 \pm 1.8$   &  
 & $826.9 \pm 3.2$ & $829.2 \pm 1.7$ & $842.9 \pm 2.8$
 & $886.9 \pm 2.0$ & $897.9 \pm 3.0$ & $987.2 \pm 0.4$ \\

ViT-L/16 (307M)
 & & $1357.9 \pm 1.9$ &
 & $1363.2 \pm 4.2$ & $1364.7 \pm 4.6$ & $1367.3 \pm 2.2$
 & $1507.6 \pm 3.1$ & $1543.3 \pm 6.3$ & $1842.7 \pm 12.9$ \\

ViT-H/14 (632M)
 & & $3252.5 \pm 3.8$ &
 & $3275.7 \pm 6.2$ & $3288.6 \pm 4.6$ & $3290.8 \pm 8.5$
 & $3544.2 \pm 3.1$ & $3598.9 \pm 6.7$ & $4068.9 \pm 7.2$ \\
\hline
\end{tabular}
}
\end{table}

\begin{table}[h]
\centering
\caption{Runtime comparison (seconds) across models and noise-buffer sizes for IMDB fine-tuning over one epoch. We use batch size $128$. LSTM, BERT-Tiny, and BERT-Base are run on an NVIDIA A100 with 64 CPU cores and physical batch size $64$; BERT-Large is run on an NVIDIA H100 with 64 CPU cores with physical batch size $16$. The standard deviations are computed over $3$ runs for each model, except for BERT-Tiny, for which we used $10$ runs to get more reliable estimates.}
\label{tab:runtime_comparison_imdb}
\resizebox{\textwidth}{!}{
\begin{tabular}{lccccccccc}
\hline
 & \multicolumn{3}{c}{\textbf{DP-SGD}} 
 & \multicolumn{3}{c}{\textbf{GPU Noise Regeneration}} 
 & \multicolumn{3}{c}{\textbf{CPU Noise Buffer}} \\
\cline{2-10}
\textbf{Model} 
 &   &   & 
 & 2 & 4 & 16 
 & 2 & 4 & 16 \\
\hline

LSTM (1M)
 & & $310.2 \pm 3.8$ &
 & $317.1 \pm 1.3$ & $317.7 \pm 5.8$ & $318.8 \pm 3.6$
 & $322.4 \pm 3.7$ & $322.8 \pm 3.1$ & $325.9 \pm 3.0$ \\

BERT-Tiny (3M)
 & & $14.2 \pm 0.6$ &
 & $14.8 \pm 0.4$ & $15.5 \pm 0.7$ & $20.6 \pm 0.8$
 & $25.2 \pm 1.4$ & $25.5 \pm 2.1$ & $29.8 \pm 2.6$ \\

BERT-Base (110M)
 & & $334.3 \pm 0.8$ &
 & $339.1 \pm 0.5$ & $342.1 \pm 1.6$ & $369.7 \pm 1.0$
 & $543.3 \pm 3.2$ & $554.4 \pm 4.8$ & $671.7 \pm 4.5$ \\

BERT-Large (340M)
 & & $499.0 \pm 0.1$ &
 & $508.7 \pm 2.0$ & $513.0 \pm 1.5$ & $539.2 \pm 1.7$
 & $846.9 \pm 3.2$ & $859.4 \pm 3.9$ & $1236.5 \pm 17.5$ \\

\hline
\end{tabular}
}
\end{table}

\begin{table}[h]
\caption{Hyperparameters for the CIFAR-10 experiments with amplification by subsampling
in Figure~\ref{fig:method_comparison_cifar_10_amplification}, for each privacy
budget $\epsilon$ with $\delta=10^{-5}$. Batch size $B$ and clipping norm
$\zeta$ are fixed. We tune the learning rate (lr) for all methods and $\lambda$
only for \method, using a public validation set. We search over
$\text{lr}\in\{0.01,0.02,0.05,0.08,0.1,0.2,0.5,1\}$ and
$\lambda\in\{0,0.1,0.25,0.5,0.8,0.9,0.95,0.9625,0.975,0.98\}$.}

\centering
\small
\setlength{\tabcolsep}{3.3pt}
\begin{tabular}{lcccccccccccccc}
\toprule
& & &
\multicolumn{2}{c}{$\epsilon=0.25$}
& \multicolumn{2}{c}{$\epsilon=0.5$}
& \multicolumn{2}{c}{$\epsilon=1$}
& \multicolumn{2}{c}{$\epsilon=2$}
& \multicolumn{2}{c}{$\epsilon=4$}
& \multicolumn{2}{c}{$\epsilon=8$} \\
\cmidrule(lr){4-5}\cmidrule(lr){6-7}\cmidrule(lr){8-9}\cmidrule(lr){10-11}\cmidrule(lr){12-13}\cmidrule(lr){14-15}
Method & $B$ & $\zeta$
& $\lambda$ & lr
& $\lambda$ & lr
& $\lambda$ & lr
& $\lambda$ & lr
& $\lambda$ & lr
& $\lambda$ & lr \\
\midrule
DP-SGD
& 128 & 8
& 0 & 0.05
& 0 & 0.08
& 0 & 0.2
& 0 & 0.2
& 0 & 0.2
& 0 & 0.2 \\

DP-$\lambda$CGD
& 128 & 8
& 0.8  & 0.05
& 0.9  & 0.08
& 0.9  & 0.2
& 0.9  & 0.2
& 0.95 & 0.5
& 0.95 & 0.5 \\

BandInvMF ($p=2$)
& 128 & 8
& 0.977 & 0.02
& 0.977 & 0.05
& 0.977 & 0.08
& 0.977 & 0.2
& 0.977 & 0.2
& 0.977 & 0.5 \\

BandInvMF ($p=16$)
& 128 & 8
& --- & 0.02
& --- & 0.08
& --- & 0.08
& --- & 0.2
& --- & 0.5
& --- & 0.5 \\

BISR ($p=2$)
& 128 & 8
& 0.5 & 0.05
& 0.5 & 0.08
& 0.5 & 0.1
& 0.5 & 0.1
& 0.5 & 0.2
& 0.5 & 0.2 \\

BISR ($p=16$)
& 128 & 8
& --- & 0.05
& --- & 0.1
& --- & 0.2
& --- & 0.2
& --- & 0.5
& --- & 0.5 \\

BLT
& 128 & 8
& --- & 0.05
& --- & 0.1
& --- & 0.2
& --- & 0.5
& --- & 0.5
& --- & 0.5 \\
\bottomrule
\end{tabular}

\label{tab:hyperparams_amplification}
\end{table}

\begin{table}[h!]
\caption{Hyperparameters for the CIFAR-10 experiments without amplification by subsampling
in Figure~\ref{fig:method_comparison_cifar_10_no_amplification}, for each privacy
budget $\epsilon$ with $\delta=10^{-5}$. Batch size $B$ and clipping norm
$\zeta$ are fixed. We tune the learning rate lr for all methods and $\lambda$
only for \method, using a public validation set. 
}
\centering
\footnotesize
\setlength{\tabcolsep}{2.3pt}
\begin{tabular}{lcccccccccccccc}
\toprule
& & &
\multicolumn{2}{c}{$\epsilon=0.25$}
& \multicolumn{2}{c}{$\epsilon=0.5$}
& \multicolumn{2}{c}{$\epsilon=1$}
& \multicolumn{2}{c}{$\epsilon=2$}
& \multicolumn{2}{c}{$\epsilon=4$}
& \multicolumn{2}{c}{$\epsilon=8$} \\
\cmidrule(lr){4-5}
\cmidrule(lr){6-7}
\cmidrule(lr){8-9}
\cmidrule(lr){10-11}
\cmidrule(lr){12-13}
\cmidrule(lr){14-15}
Method & $B$ & $\zeta$
& $\lambda$ & lr
& $\lambda$ & lr
& $\lambda$ & lr
& $\lambda$ & lr
& $\lambda$ & lr
& $\lambda$ & lr \\
\midrule

DP-SGD
& 128 & 8
& 0 & 0.002
& 0 & 0.005
& 0 & 0.01
& 0 & 0.02
& 0 & 0.05
& 0 & 0.05 \\

DP-$\lambda$CGD
& 128 & 8
& 0.9625 & 0.005
& 0.9625 & 0.02
& 0.95   & 0.05
& 0.95   & 0.1
& 0.95   & 0.2
& 0.95   & 0.5 \\

BandInvMF ($p=2$)
& 128 & 8
& 0.977 & 0.01
& 0.977 & 0.02
& 0.977 & 0.05
& 0.977 & 0.08
& 0.977 & 0.1
& 0.977 & 0.5 \\

BandInvMF ($p=16$)
& 128 & 8
& --- & 0.01
& --- & 0.01
& --- & 0.1
& --- & 0.2
& --- & 0.2
& --- & 0.5 \\

BISR ($p=2$)
& 128 & 8
& 0.5 & 0.005
& 0.5 & 0.005
& 0.5 & 0.02
& 0.5 & 0.02
& 0.5 & 0.05
& 0.5 & 0.1 \\

BISR ($p=16$)
& 128 & 8
& --- & 0.01
& --- & 0.02
& --- & 0.05
& --- & 0.1
& --- & 0.2
& --- & 0.5 \\

BLT
& 128 & 8
& --- & 0.01
& --- & 0.02
& --- & 0.05
& --- & 0.1
& --- & 0.2
& --- & 0.5 \\

\bottomrule
\end{tabular}
\label{tab:hyperparams_no_amplification}
\end{table}

\clearpage
\section{Proofs}

\MaxSEBound*
\begin{proof}

For any factorization, MaxSE is at least as large as RMSE. Therefore, it suffices to prove the upper bound $O(k + \sqrt{k}n^{1/4})$ for MaxSE and the lower bound $\Omega(k + \sqrt{k}n^{1/4})$ for RMSE. We begin with the upper bound.

First, we compute the squared sensitivity of the matrix $C_{\lambda}$:
\begin{align}
    \mathrm{sens}_{k,b}^2(C_{\lambda})
    &= \left\|\sum\limits_{j = 0}^{k - 1}(C_\lambda)_{:, jb + 1}\right\|_2^2
    = \sum\limits_{j = 0}^{k - 1}\sum\limits_{r = 0}^{b - 1}
      \left(\sum\limits_{m = 0}^{j}\lambda^{jb + r - mb}\right)^2  \\
    &= \sum\limits_{j = 0}^{k - 1}\sum\limits_{r = 0}^{b - 1}
      \lambda^{2r} \left(\frac{1 - \lambda^{b(j + 1)}}{1 - \lambda^b}\right)^2
     = \frac{1 - \lambda^{2b}}{(1- \lambda^2)(1 - \lambda^b)^2}\sum\limits_{j = 1}^{k}(1 - \lambda^{bj})^2.
\end{align}

The squared maximum row $\ell_2$-norm of the matrix $AC_{\lambda}^{-1}$ is
\begin{equation}
    \|AC_{\lambda}^{-1}\|_{2\to \infty}^2 = 1 + (1 - \lambda)^2(n - 1).
\end{equation}

Combining these expressions gives the squared maximum error. To prove the claimed upper bound, we consider two regimes.

First, suppose $k \ge \sqrt{n}$ and set $\lambda = e^{-1/b}$. This yields
\begin{align}
    \|AC_{\lambda}^{-1}\|_{2\to \infty}^2 \cdot \mathrm{sens}_{k,b}^2(C_{\lambda}) &= \left(1 + \frac{n - 1}{b^2} + O(nb^{-3})\right) \left(\frac{(1 - e^{-2})(k + O(1))}{(2b^{-1} + O(b^{-2}))(1 - e^{-1})^2} \right)\\
    &= O\left(\frac{nkb}{b^2}\right) = O(k^2),
\end{align}
assuming $n=kb$.

If $k < \sqrt{n}$, take $\lambda = e^{-1/\sqrt{n}}$. Then $\lambda^b = e^{-\sqrt{n}/k} < 1/e = O(1)$ and $1 -\lambda^{b} = \Theta(1)$, 
\begin{equation}
    \|AC_{\lambda}^{-1}\|_{2\to \infty}^2 \cdot \mathrm{sens}_{k,b}^2(C_{\lambda}) = O(1) \left(\frac{O(1)}{2/\sqrt{n} + O(1/n)}O(k)\right) = O(k\sqrt{n}).
\end{equation}

Combining the two bounds and taking square roots yields the claimed upper bound.

For the lower bound on RMSE, there is a universal $\Omega(k)$ lower bound that holds for any factorization \citep{kalinin2025back}.
 It remains only to prove the $\Omega(\sqrt{k}n^{1/4})$ lower bound.
We first lower-bound the sensitivity:
\begin{equation}
\begin{aligned}
     \mathrm{sens}_{k,b}^2(C_{\lambda})
    &= \left\|\sum\limits_{j = 0}^{k - 1}(C_\lambda)_{:, jb + 1}\right\|_2^2
    \ge \sum\limits_{j = 0}^{k - 1} \|(C_{\lambda})_{:, jb + 1}\|^2_2
    \ge\sum\limits_{j = 0}^{\lceil k/2 \rceil - 1} \|(C_{\lambda})_{:, jb + 1}\|^2_2\\
    &\ge \lceil k/2 \rceil \cdot \|(C_{\lambda})_{:, \lceil k / 2 \rceil b - b + 1}\|^2_2 = \lceil k / 2 \rceil \sum\limits_{j = 0}^{n - \lceil k / 2 \rceil b +b-1} \lambda^{2j}\\
    &= \lceil k / 2 \rceil \frac{1 - \lambda^{2n - 2\lceil k/2 \rceil b+ 2b} }{1 - \lambda^2} \ge \lceil k / 2 \rceil \frac{1 - \lambda^{n}}{1 - \lambda^2}
    \ge \frac{k}{4} \frac{1 - \lambda^{n}}{1 - \lambda}.
\end{aligned}
\end{equation}

Next, we prove the following auxiliary inequality for any $\lambda \in [0,1)$ and $n$:
\begin{equation}
    1 - \lambda^{n} \ge \frac{n(1 - \lambda)}{1 + n(1 - \lambda)}.
\end{equation}
This follows from Bernoulli's inequality:
\begin{equation}
    \lambda^{n}
    = (1 - (1 - \lambda))^{n}
    \le \frac{1}{(1 + (1 - \lambda))^{n}}
    \le \frac{1}{1 + n(1 -\lambda)}.
\end{equation}

Multiplying the lower bound on the sensitivity by $\frac{1}{n}\|AC_{\lambda}^{-1}\|_{F}^2$ gives
\begin{equation}
\begin{aligned}
    \inf_{\lambda \in [0,1)} \tfrac{1}{n}\|AC_{\lambda}^{-1}\|_{F}^2
    \cdot \mathrm{sens}_{k,b}^2(C_{\lambda})
    &\ge \inf_{\lambda \in [0,1)}
    \frac{kn(1 + (1 - \lambda)^2 (n -1)/2)}{4(1 + n(1 - \lambda))} \\
    &\ge \frac{kn}{4}\inf_{\lambda \in [0,1)}
    \frac{1 + (1 - \lambda)^2(n - 1)/2}{1 +n(1 - \lambda)}.
\end{aligned}
\end{equation}

To lower-bound the infimum, we consider two cases. First, suppose that
$1 - \lambda \le \frac{1}{\sqrt{n}}$. Then
\begin{equation}
    \frac{1 + (1 - \lambda)^2(n - 1)/2}{1 +n(1 - \lambda)}
    \ge \frac{1}{1 +n(1 - \lambda)}
    \ge \frac{1}{1 + \sqrt{n}}
    \ge \frac{1}{8\sqrt{n}}.
\end{equation}

Second, suppose that
$1 - \lambda \ge \frac{1}{\sqrt{n}}$. Then, for $n\ge2$,
\begin{equation}
\begin{aligned}
    \frac{1 + (1 - \lambda)^2(n - 1)/2}{1 +n(1 - \lambda)}
  \ge \frac{(1 - \lambda)^2(n - 1)}{4n(1 - \lambda)}
    = \frac{(1-\lambda)(n-1)}{4n}
    \ge \frac{1 - \lambda}{8}
    \ge \frac{1}{8\sqrt{n}}.
\end{aligned}
\end{equation}

Therefore,
\begin{equation}
    \inf_{\lambda \in [0,1)}
    \tfrac{1}{n}\|AC_{\lambda}^{-1}\|_{F}^2
    \cdot \mathrm{sens}_{k,b}^2(C_{\lambda})
    \ge \frac{kn}{32\sqrt{n}}
    = \frac{k\sqrt{n}}{32}
    = \Omega(k \sqrt{n}).
\end{equation}
Taking square roots yields the $\Omega(\sqrt{k}n^{1/4})$ lower bound. Combining this with the $\Omega(k)$ lower bound concludes the proof.
\end{proof}

\MaxSEvsRMSEMinimizerInequality*
\begin{proof}
Define, for $\gamma>0$,
\begin{equation}
    F_{\gamma}(\lambda) = \bigl(1 + (1-\lambda)^2 (n-1)\gamma\bigr)\,\sigma_{\epsilon, \delta}^2(C_{\lambda}),
\end{equation}
where $\sigma_{\epsilon, \delta}(C_{\lambda})$ is the noise multiplier in the amplified setting and sensitivity in the non-amplified setting.
Then squared MaxSE is given by $F_1(\lambda)$ and squared RMSE is given by $F_{1/2}(\lambda)$.
We claim that $\frac{F_1(\lambda)}{F_{1/2}(\lambda)}$ is a monotonically decreasing function of $\lambda$.
Indeed, for $\lambda\in(0,1)$:
\begin{equation}
\frac{d}{d\lambda}\left(\frac{F_1(\lambda)}{F_{1/2}(\lambda)}\right)
= \frac{d}{d\lambda}\left(\frac{1 + (1-\lambda)^2(n-1)}{1 + (1-\lambda)^2(n-1)/2}\right)
= -\frac{(n-1)(1-\lambda)}{\left(1+(1-\lambda)^2(n-1)/2\right)^2} < 0.
\end{equation}

Let $\lambda_1$ be a minimizer of $F_1(\lambda)$ and let $\lambda_{1/2}$ be a minimizer of $F_{1/2}(\lambda)$.
Assume for contradiction that $\lambda_1 < \lambda_{1/2}$.
Since $\lambda_1$ and $\lambda_{1/2}$ are minimizers and $\sigma_{\epsilon, \delta}(C_{\lambda})>0$, we have
\begin{equation}
    F_{1}(\lambda_{1}) \le F_1(\lambda_{1/2})
    \quad\text{and}\quad
    F_{1/2}(\lambda_{1/2}) \le F_{1/2}(\lambda_{1}).
\end{equation}
Combining these two inequalities gives 
\begin{equation}
    \frac{F_{1}(\lambda_{1})}{F_{1/2}(\lambda_{1})}
    \le
    \frac{F_{1}(\lambda_{1/2})}{F_{1/2}(\lambda_{1/2})}.
\end{equation}
However, the ratio $\frac{F_1(\lambda)}{F_{1/2}(\lambda)}$ is decreasing in $\lambda$, and $\lambda_1 < \lambda_{1/2}$ would imply the reverse inequality, which is a contradiction.
Therefore $\lambda_1 \ge \lambda_{1/2}$.
\end{proof}

\FullBatchConstants*
\begin{proof}
By the definition of RMSE (Equation~\ref{eq:RMSE_definition}), for any factorization $A = BC$, we have
\begin{equation}
\mathrm{RMSE}(B, C) = \frac{1}{\sqrt{n}} \|B\|_F \,\mathrm{sens}_{n,1}(C), \quad \text{for} \quad
\mathrm{sens}_{n,1}(C) = \sup_{X \sim X'} \|CX - CX'\|_F.
\end{equation}
Here, neighboring datasets may differ in any coordinates, and for each $i$, $\|X_i - X'_i\|_2 \le 1$.
To obtain a lower bound, take $X$ and $X'$ that differ in all coordinates
and lie along a single dimension. Thus,
\begin{equation}
\mathrm{sens}_{n,1}(C) \ge \|C\|_{\infty \to 2} = \sup_{\|x\|_{\infty} \le 1} \|Cx\|_2.
\end{equation}
It would then suffice to prove that
\begin{equation}
\inf_{BC=A} \frac{1}{\sqrt{n}} \|B\|_F \|C\|_{\infty \to 2}
= \frac{1}{\sqrt{n}} \sum\limits_{j = 1}^{n} \sqrt{j}.
\end{equation}

Let $u_j = \frac{Ae_j}{\|Ae_j\|_2}$ and consider the matrix $U=[u_1,\dots,u_n]$. Then
\begin{equation}
\sum\limits_{j = 1}^{n}\sqrt{j} = \sum\limits_{j=1}^{n} \|Ae_j\|_2
= \mathrm{tr}(U^\top A)
= \mathrm{tr}(U^\top BC)
= \mathrm{tr}(BCU^\top)
\le \|B\|_F \|UC^\top\|_F.
\end{equation}

It is thus sufficient to show that $\|UC^\top\|_F \le \|C\|_{\infty \to 2}$. For this, we use a probabilistic argument.

Let $z$ be a random vector such that $\mathbb{E}[zz^\top] = U^\top U$ and $\|z\|_\infty \le 1$. Then

\begin{equation}
\begin{aligned}
\|C\|_{\infty \to 2}^2
&= \sup_{\|x\|_{\infty} \le 1} \|Cx\|_2^2
\ge \mathbb{E}\|Cz\|_2^2
= \mathbb{E}[\mathrm{tr}(Czz^\top C^\top)]
= \mathrm{tr}\!\left(C^\top C\,\mathbb{E}[zz^\top]\right)\\
&= \mathrm{tr}(C^\top C U^\top U)
= \|UC^\top\|_F^2.
\end{aligned}
\end{equation}

It remains only to show that such a random vector $z$ exists. The entries of the matrix $U^\top U$ are given by
\begin{equation}
(U^\top U)_{i,j}
= \langle u_i,u_j\rangle
= \frac{n-\max(i,j)+1}{\|Ae_i\|_2\|Ae_j\|_2}
= \sqrt{\frac{n-\max(i,j)+1}{n-\min(i,j)+1}}
\le 1.
\end{equation}

Consider auxiliary independent random variables $y_1,\dots,y_{n-1}$ taking values in $\{-1,1\}$, with the probability of taking the value $1$ given by

\begin{equation}
\mathbb{P}(y_j=1)
= \frac{1+\sqrt{\frac{n-j}{n-j+1}}}{2},
\qquad
\text{and therefore}
\qquad
\mathbb{E}[y_j]
= \sqrt{\frac{n-j}{n-j+1}}.
\end{equation}

Let $z_1$ be uniformly distributed on $\{-1,1\}$, independently of $y_1,\dots,y_{n-1}$, and let $z_j=z_1\prod_{i=1}^{j-1}y_i$. Then $\|z\|_\infty=1$ by construction, while
\begin{equation}
\begin{aligned}
\mathbb{E}[z_i z_j]
&= \mathbb{E}\left[z_1^2\prod_{l=1}^{i-1}y_l\prod_{r=1}^{j-1}y_r\right]
= \mathbb{E}\left[\prod_{l=\min(i,j)}^{\max(i,j)-1}y_l\right]
= \prod_{l=\min(i,j)}^{\max(i,j)-1}\mathbb{E}[y_l]\\
&= \sqrt{\frac{n-\max(i,j)+1}{n-\min(i,j)+1}}
= \langle u_i,u_j\rangle.
\end{aligned}
\end{equation}

Consequently,
\begin{equation}
\frac{1}{\sqrt{n}}\|B\|_F \|C\|_{\infty \to 2}
\ge \frac{1}{\sqrt{n}}\sum\limits_{j=1}^{n} \|Ae_j\|_2
= \frac{1}{\sqrt{n}} \sum\limits_{j=1}^{n} \sqrt{n-j+1}
= \frac{1}{\sqrt{n}} \sum\limits_{j=1}^{n} \sqrt{j}.
\end{equation}

If equality holds, then
\begin{equation}
\|C\|_{\infty \to 2}^2
= \sup_{\|x\|_{\infty}\le 1}\|Cx\|_2^2
= \mathbb{E}\|Cz\|_2^2.
\end{equation}
The vector $z$ has non-zero probability of taking every possible assignment in $\{-1,1\}^n$. Since its expected value attains the largest possible value of $\|Cz\|_2^2$, every assignment of $z$ must give the same value. This is possible only if$(C^\top C)_{i,j}=0$, for  $i\ne j$, hence $C^\top C$ is diagonal.

Since $C$ is lower triangular and invertible, this implies that $C$ itself is diagonal. Indeed, the last column of $C$ contains only the non-zero entry $C_{n,n}$, and its orthogonality to every preceding column implies $C_{n,j}=0$ for all $j<n$. Repeating the argument recursively for the leading principal submatrix shows that all off-diagonal entries of $C$ vanish.

For a diagonal matrix $C$, it follows from Proposition E.1 in \citet{choquette2023amplified} that $\mathrm{sens}_{n,1}(C) = \|Ce\|_2$. To show the matching upper bound, we use the following explicit diagonal construction:
\begin{equation}
C_{\mathrm{diag}} = \mathrm{diag}\bigl(n^{1/4}, (n-1)^{1/4}, \ldots, 1^{1/4}\bigr),
\qquad
(B_{\mathrm{diag}})_{i,j} =
\begin{cases}
(n - j + 1)^{-1/4}, & i \ge j,\\
0, & \text{o/w}.
\end{cases}
\end{equation}
Then
\vspace{-0.3cm}
\begin{equation}
\|C_{\mathrm{diag}}e\|_2
= \sqrt{\sum_{j=1}^{n} \sqrt{j}}, \quad
\|B_{\mathrm{diag}}\|_F
= \sqrt{
\sum_{j=1}^{n}
(n - j + 1) \cdot (n - j + 1)^{-1/2}
}
=
\sqrt{\sum_{j=1}^{n} \sqrt{j}}.
\end{equation}
Thus,
\vspace{-0.3cm}
\begin{equation}
\label{eq:rmse_b_diag_c_diag}
\mathrm{RMSE}(B_{\mathrm{diag}}, C_{\mathrm{diag}})=\frac{1}{\sqrt{n}}\|B_{\mathrm{diag}}\|_F \|C_{\mathrm{diag}}e\|_2
=
\frac{1}{\sqrt{n}} \sum_{j=1}^{n} \sqrt{j},
\end{equation}
showing that the stated lower bound is attained. Finally, by the Cauchy--Schwarz inequality, this diagonal factorization is optimal among all diagonal choices of $C$. Thus, if $C$ is diagonal and equality is attained, then $C$ must be a positive scalar multiple of $C_{\mathrm{diag}}$.

\paragraph{The error of the trivial factorization.}

The trivial factorization $C = I$ yields
\begin{equation}
    \mathrm{RMSE}(A, I)
    \;=\;
    \sqrt{\frac{n(n+1)}{2}}
    \;\sim\;
    \frac{n}{\sqrt{2}}.
\end{equation}
This is approximately $6\%$ larger than the error of the optimal factorization with $C=C_{\mathrm{diag}}$, whose error is
\begin{equation}
\mathrm{RMSE}(B_{\mathrm{diag}}, C_{\mathrm{diag}}) = \frac{1}{\sqrt{n}} \sum\limits_{j = 1}^{n} \sqrt{j} \sim \frac{2n}{3}.
\end{equation}

\end{proof}

\trivialfactorizationcorollary*
\begin{proof}

For the MaxSE metric, \citet[Proposition~3.11]{pillutla2025correlated} show that, in dimension $d=1$, the trivial factorization $B=A$, $C=I$ is optimal among all factorizations when $b=1$ and $k=n$. Since the sensitivity $\mathrm{sens}_{n,1}(C)$ for $d>1$ is at least as large as in dimension $d=1$, the same lower bound applies in any dimension. Hence, DP-SGD is MaxSE-optimal for all $d$ in the full-batch regime. In particular, the trivial factorization is also optimal within the \method family under the MaxSE metric.

Finally, Lemma~\ref{lem:method_minimizer_inequality} shows that the RMSE-optimal value of $\lambda$ is no larger than the MaxSE-optimal one. Since $\lambda\geq 0$ and the latter is $0$ in the full-batch regime, the RMSE-optimal value must also be $\lambda=0$. Therefore, the trivial factorization is optimal for \method under both the RMSE and MaxSE metrics.

\end{proof}

\NormalizedSensitivity*

\begin{proof}
We first note that the matrix $C_\lambda D^{-1}$ is no longer Toeplitz, and therefore Theorem~\ref{thm:sensitivity} does not apply directly. We instead repeat the argument of the proof of that theorem for this specific matrix.

For an element-wise nonnegative matrix $C$, the sensitivity can be written \citep{choquette2023amplified} as
\begin{equation}
    \mathrm{sens}(C)
    = \sup_{\pi} \sqrt{\sum_{i,j \in \pi} \langle C_{:,i}, C_{:,j} \rangle},
\end{equation}
where $\pi$ ranges over sets of at most $k$ column indices with minimum separation $b$.

Following the argument of \citet{kalinin2024}, we show that the maximizing set $\pi$ consists of the leftmost admissible columns. Specifically, if the optimal indices are not the leftmost ones (subject to $b$-separation), then the entire block can be shifted left without decreasing the objective. It suffices to show two facts:
\begin{enumerate}
    \item If $i+b<j$, then
    \[
        \langle (C_\lambda D^{-1})_{:,i}, (C_\lambda D^{-1})_{:,j} \rangle
        \le
        \langle (C_\lambda D^{-1})_{:,i}, (C_\lambda D^{-1})_{:,j-1} \rangle .
    \]
    \item If $i>1$ and $i+b \le j$, then
    \[
        \langle (C_\lambda D^{-1})_{:,i}, (C_\lambda D^{-1})_{:,j} \rangle
        \le
        \langle (C_\lambda D^{-1})_{:,i-1}, (C_\lambda D^{-1})_{:,j-1} \rangle .
    \]
\end{enumerate}
Since all columns are normalized to unit norm, these properties imply that the maximum sensitivity is attained by choosing the leftmost $k$ columns with separation $b$, i.e., indices $\{1,\, 1+b,\, \dots,\, 1+(k-1)b\}$.

We now verify the two claims. The inner product between the $i$th and $j$th columns of $C_\lambda D^{-1}$ is
\begin{equation}
    \langle (C_\lambda D^{-1})_{:,i}, (C_\lambda D^{-1})_{:,j} \rangle
    =
    \frac{1}{d_i d_j}
    \sum_{t=0}^{n-j} \lambda^{t} \lambda^{j-i+t},
\end{equation}
where
\begin{equation}
    d_i
    = \sqrt{\sum_{t=0}^{n-i} \lambda^{2t}}
    = \sqrt{\frac{1-\lambda^{2(n-i+1)}}{1-\lambda^2}} .
\end{equation}

Similarly, the inner product between columns $i$ and $j-1$ is
\begin{equation}
    \langle (C_\lambda D^{-1})_{:,i}, (C_\lambda D^{-1})_{:,j-1} \rangle
    =
    \frac{1}{d_i d_{j-1}}
    \sum_{t=0}^{n-j+1} \lambda^{t} \lambda^{j-i-1+t}.
\end{equation}
Taking the difference yields
\begin{align}
    &\langle (C_\lambda D^{-1})_{:,i}, (C_\lambda D^{-1})_{:,j-1} - (C_\lambda D^{-1})_{:,j} \rangle \\
    &\qquad
    = \frac{\lambda^{j-i-1}}{d_i d_{j-1}} d_{j-1}^2
      - \frac{\lambda^{j-i}}{d_i d_j} d_j^2
    = \frac{\lambda^{j-i-1}}{d_i} (d_{j-1} - \lambda d_j) \ge 0,
\end{align}
since the sequence $(d_j)$ is decreasing. This proves the first claim.

For the second claim, we consider
\begin{align}
    &\langle (C_\lambda D^{-1})_{:,i-1}, (C_\lambda D^{-1})_{:,j-1} \rangle
    - \langle (C_\lambda D^{-1})_{:,i}, (C_\lambda D^{-1})_{:,j} \rangle \\
    &\qquad
    = \lambda^{j-i}
    \left(
        \frac{d_{j-1}}{d_{i-1}} - \frac{d_j}{d_i}
    \right).
\end{align}
Thus it suffices to show that
\[
    \frac{d_{j-1} d_i}{d_{i-1} d_j} \ge 1.
\]
Using the explicit form of $d_i$, we obtain
\begin{equation}
    \frac{d_{j-1} d_i}{d_{i-1} d_j}
    =
    \sqrt{
        \frac{
            (1-\lambda^{2(n-j+2)})(1-\lambda^{2(n-i+1)})
        }{
            (1-\lambda^{2(n-i+2)})(1-\lambda^{2(n-j+1)})
        }
    }.
\end{equation}
Expanding the numerator and denominator, we see that this ratio is at least one if and only if
\begin{equation}
    \lambda^{2(n - j + 1)} + \lambda^{2(n - i + 2)}
    -
     \lambda^{2(n - j + 2)} - \lambda^{2(n - i + 1)}
    \ge 0.
\end{equation}
Rearranging terms gives
\begin{equation}
    (1 - \lambda^2)
    \bigl(
        \lambda^{2(n - j + 1)} - \lambda^{2(n - i + 1)}
    \bigr)
    \ge 0,
\end{equation}
which holds whenever $i \le j$. This concludes the proof.
\end{proof}

\rmseComparisonLemma*

\begin{proof}
The strategy matrix of \method is
\begin{equation}
    C_{\lambda} = \begin{pmatrix}
        1 & 0 & \cdots &0\\
        \lambda & 1 & \cdots & 0\\
        \vdots & \vdots & \ddots & \vdots\\
        \lambda^{n - 1} & \lambda^{n - 2} & \cdots & 1
    \end{pmatrix}.
\end{equation}

The squared norms of its columns are
\begin{equation}
    d_j^2 = \sum\limits_{i = 0}^{n - j} \lambda^{2i}
    = \frac{1 - \lambda^{2(n - j + 1)}}{1 -\lambda^2}. 
\end{equation}

The sensitivity of $C_{\lambda}$ under single participation is $d_1$.
For the factorization $A = B_{\lambda} C_{\lambda}$, where $A$ is the prefix-sum matrix, the corresponding left matrix is
\begin{equation}
    B_{\lambda} = \begin{pmatrix}
        1 & 0 & 0 & \cdots & 0\\
        1- \lambda & 1 & 0 & \dots &0\\
        1 - \lambda & 1- \lambda & 1 &\cdots & 0\\
        \vdots & \vdots & \vdots & \ddots & \vdots\\
        1 - \lambda & 1- \lambda & 1 - \lambda & \cdots & 1
    \end{pmatrix}.
\end{equation}

Its Frobenius norm can be written as
\begin{equation}
    \|B_{\lambda}\|_{F}^2 = (1 - \lambda)^2 \frac{(n - 1)n}{2} + n.
\end{equation}

The sensitivity of the normalized \method is equal to $1$ by construction.
The normalized strategy matrix is $\tilde{C}_{\lambda} = C_{\lambda} D^{-1}$, and the corresponding left matrix is
$\tilde{B}_{\lambda} = A D C_{\lambda}^{-1}$.
Its entries are
\begin{equation}
    (\tilde{B}_{\lambda})_{i,j} = \begin{cases}
        d_{j} - \lambda d_{j + 1} \quad & \text{if} \quad j < i,\\
        d_j \quad &\text{if} \quad j = i,\\
        0 \quad &\text{otherwise.}
    \end{cases}
\end{equation}
Therefore, its Frobenius norm satisfies
\begin{equation}
    \|\tilde{B}_{\lambda}\|_F^2
    = \sum\limits_{j = 1}^{n} d_j^2
    + \sum\limits_{j = 1}^{n - 1} (d_j - \lambda d_{j + 1})^2 (n - j).
\end{equation}

To show that the normalized \method has lower RMSE than \method, it is sufficient to prove
\begin{equation}
\label{ineq:rmse_norm_method}
    \sum\limits_{j = 1}^{n} d_j^2
    + \sum\limits_{j = 1}^{n - 1} (d_j - \lambda d_{j + 1})^2 (n - j)
    \le d_1^2 \left(n + (1 - \lambda)^2 \frac{(n - 1)n}{2}\right).
\end{equation}

First, we compute the sum of $d_j^2$:
\begin{equation}
\label{eq:d_j_squared_sum}
    \sum\limits_{j = 1}^{n} d_j^2
    = \sum\limits_{j = 1}^{n}\sum\limits_{i = 0}^{n - j }\lambda^{2i}
    = \sum\limits_{i = 0}^{n - 1}\lambda^{2i} (n - i)
    = n d_1^2 - \sum\limits_{j = 1}^{n - 1} j\lambda^{2j}.
\end{equation}

Next, we bound $d_j - \lambda d_{j + 1}$ as follows:
\begin{equation}
\begin{aligned}
    (d_j - \lambda d_{j + 1})^2
    &= d_j^2 +\lambda^2 d_{j + 1}^2 - 2\lambda d_j d_{j + 1}\\
    &= d_{j + 1}^2 + \lambda^{2(n - j)} +\lambda^2 d_{j + 1}^2 -2\lambda d_j d_{j + 1}\\
    &\le (1 -\lambda)^2 d_{j + 1}^2 + \lambda^{2(n - j)}.
\end{aligned}
\end{equation}

Thus,
\begin{align}
    \sum\limits_{j = 1}^{n - 1}(d_j - \lambda d_{j + 1})^2 (n - j)
    &\le (1-\lambda)^2 \sum\limits_{j = 1}^{n -1} d_{j + 1}^2 (n -j)
      + \sum\limits_{j = 1}^{n - 1} j \lambda^{2j} \\
    &\le d_1^2 (1 - \lambda)^2 \sum\limits_{j = 1}^{n - 1} (n- j)
      + \sum\limits_{j = 1}^{n - 1} j \lambda^{2j}\\
    &= d_1^2 (1 - \lambda)^2 \frac{n ( n - 1)}{2 }
      + \sum\limits_{j = 1}^{n - 1} j \lambda^{2j}.
\end{align}

Combining this bound with \eqref{eq:d_j_squared_sum} yields \eqref{ineq:rmse_norm_method}.
\end{proof}

\end{document}